\documentclass{article} % For LaTeX2e
\usepackage{iclr2027_conference,times}

\usepackage{amsmath,amsfonts,bm}

\def\eqref#1{equation~\ref{#1}}
\def\1{\bm{1}}

\DeclareMathAlphabet{\mathsfit}{\encodingdefault}{\sfdefault}{m}{sl}
\SetMathAlphabet{\mathsfit}{bold}{\encodingdefault}{\sfdefault}{bx}{n}

\usepackage{hyperref}
\usepackage{url}
\usepackage{amsmath,amssymb,amsthm}
\usepackage{booktabs}
\usepackage{multirow}
\usepackage{graphicx}
\usepackage{enumitem}

\usepackage{float}
\usepackage{algorithm}
\usepackage{algpseudocode}
\usepackage{import}

\newtheorem{theorem}{Theorem}[section]
\newtheorem{proposition}[theorem]{Proposition}
\newtheorem{lemma}[theorem]{Lemma}

\newtheorem{assumption}[theorem]{Assumption}

\usepackage{subcaption}

\title{Refinement-Based Flow Policy Optimization}

\author{
Bumgeun Park\thanks{Equal contribution.}~~,
Hyukjun Yang\footnotemark[1]~~,
Donghwan Lee\thanks{Corresponding author.}~~ \\
Korea Advanced Institute of Science and Technology (KAIST)\\
\texttt{\{j4t123,jundol32,donghwan\}@kaist.ac.kr}
}

\iclrfinalcopy % Uncomment for camera-ready version, but NOT for submission.
\begin{document}

\maketitle

\begin{abstract}
Flow-based policies offer an expressive representation for online reinforcement learning, but conventional flow matching requires samples drawn from the distribution to be modeled. This poses a challenge when the desired action distribution is defined only implicitly by a Q-function, since directly sampling actions from the resulting distribution is generally intractable. We propose Refinement-Based Flow Policy Optimization (RFPO), a novel framework for training a flow policy in online reinforcement learning by alternating between Q-guided sample refinement and self-target flow matching. RFPO first generates actions from Gaussian noise using the current flow policy and then uses a finite-step stochastic refinement procedure to move them toward an energy-based distribution induced by the Q-function. Each refined action is then paired with its corresponding initial noise sample and used as a fixed target for flow-matching training. By repeatedly refining its own outputs and learning from the resulting targets, RFPO incorporates Q-guidance into the policy without requiring direct samples from the target distribution, while retaining the capacity to represent multiple action modes. We further provide a theoretical analysis of the distributional dynamics induced by RFPO. Experiments on six synthetic two-dimensional target distributions with diverse geometries demonstrate that alternating Q-guided sample refinement and self-target flow matching captures complex multimodal structure without mode collapse. Across six continuous-control environments in MuJoCo, RFPO shows promising performance against representative online RL baselines in almost every environment.
\end{abstract}

\section{Introduction}
\label{sec:intro}
Online reinforcement learning (RL) aims to learn decision-making policies through interaction with an environment. In continuous-action actor--critic methods, a Q-function provides a learning signal for policy improvement, but the policy is often restricted to a Gaussian parameterization, which can limit its representational flexibility due to the unimodal nature of a single Gaussian. Generative modeling approaches, such as diffusion and flow matching (FM), offer promising alternatives to Gaussian policies because they can capture complex, potentially multimodal action distributions \citep{ho2020denoising,lipman2023flow}.

However, conventional diffusion and FM training relies on samples from the distribution to be modeled. In online RL, the desired action distribution is instead defined implicitly by the Q-function and changes as the critic is updated, making direct target samples difficult to obtain. Recent work has addressed this challenge through Q-weighted objectives and critic-derived score or velocity targets \citep{ding2024diffusion,ma2025efficient,psenka2024learning,li2026reverse}. Other approaches explicitly modify action samples using critic information before fitting a generative policy to the resulting targets \citep{yang2023policy,zhu2026real}. 

In this work, we address this challenge using Markov chain Monte Carlo (MCMC), which can refine initial samples toward a target distribution when direct sampling is intractable. We view this MCMC-based refinement as a sample-based policy improvement step that moves action samples toward a potentially multimodal target distribution induced by the Q-function. Building on this view, we propose Refinement-Based Flow Policy Optimization (RFPO), a novel framework for training a flow policy in online RL through iterative sample refinement.

Specifically, RFPO alternates between Q-guided sample refinement and self-target FM. The first process generates action samples from Gaussian noise using the current flow policy and applies a finite number of steps of the Metropolis-adjusted Langevin algorithm (MALA), a gradient-based MCMC method, to refine the samples toward an energy-based action distribution induced by the Q-function. In the second process, each refined action is paired with the base noise used to generate its original action, and the flow policy is trained on these pairs through conditional FM. By separating Q-guided target construction from flow-policy optimization, RFPO incorporates critic guidance through fixed regression targets. During the FM update, this construction does not require differentiating through the iterative flow sampler. To the best of our knowledge, RFPO is the first method to train a flow policy in online RL by iteratively refining samples generated by the current policy and using the refined samples as self-generated training targets. Repeatedly alternating between sample refinement and self-target FM allows the policy to learn from its own refined outputs and move toward the target distribution.

The main contributions of this work are summarized as follows:
\begin{itemize}
    \item We introduce RFPO, a novel framework for training a flow policy in online RL through two alternating processes: MCMC-based refinement of actions sampled from the current policy toward an energy-based target distribution induced by the Q-function, and self-target FM using the refined actions as training targets.

    \item We theoretically characterize the idealized distributional dynamics induced by RFPO, establishing monotonic improvement of the entropy-regularized policy objective for a fixed critic and identifying the energy-based target distribution as a fixed point.

    \item We empirically demonstrate that alternating Q-guided sample refinement and self-target FM captures complex multimodal structure across synthetic target distributions with diverse geometries, and that RFPO achieves promising performance against representative online RL baselines in almost every evaluated continuous-control environment.
\end{itemize}

\section{Related work}
\label{sec:related}

\subsection{Q-Guided Generative Policy Learning}
\label{sec:rw_q_guided_policy}
Recent studies have explored diffusion- and flow-based policies in online RL using Q-guided training objectives. For diffusion-based policies, QVPO uses a Q-weighted variational objective without requiring samples drawn directly from the target action distribution \citep{ding2024diffusion}. DPMD and SDAC employ score matching reweighted by Q-values to derive tractable updates for policy mirror descent and maximum-entropy policy optimization, respectively \citep{ma2025efficient}. Beyond Q-value weighting, other approaches directly use the gradient of the Q-function to guide policy learning. QSM derives policy updates from the relationship between the policy score and the gradient of the Q-function \citep{psenka2024learning}, whereas DACER backpropagates the gradient of the Q-function through the reverse diffusion process to update the policy and estimates policy entropy to regulate exploration \citep{wang2024diffusion}. For flow-based policy learning, RFM uses Q-derived importance weights to estimate velocity targets and the gradient of the Q-function to stabilize this estimation \citep{li2026reverse}.

A complementary line of work constructs training targets by explicitly refining action samples using the gradient of the Q-function. DIPO refines replay-buffer actions through Q-gradient ascent and fits a diffusion policy to the resulting actions \citep{yang2023policy}. DACER-F applies unadjusted Langevin updates to replay-buffer actions and trains a flow policy by combining FM objective on the refined actions with direct Q-value maximization \citep{zhu2026real}. Distinct from these methods, RFPO initializes MALA from actions sampled from the current flow policy and pairs each refined action with its corresponding source noise for self-target FM. This recursive construction enables a direct analysis of the resulting distributional dynamics, presented in Section~\ref{sec:theory}. The methods discussed above are summarized in Table~\ref{tab:positioning}.

\subsection{Amortizing Iterative Sample Refinement}
\label{sec:rw_sample_refinement}
Beyond methods that refine previously collected actions, a broader line of work repeatedly refines samples produced by the model being trained and uses the refined samples to update that same model. For example, Amortized MCMC draws initial samples from an approximation network, applies MCMC transitions to them, and trains the network to match the resulting refined distribution \citep{li2017approximate}. Amortized SVGD follows a similar structure by applying Stein variational gradient descent to samples produced by a stochastic neural network and adjusting the network parameters to reproduce the resulting refined samples \citep{feng2017learning}. MCMC teaching uses a generative model to produce initial samples from latent variables and an energy-based model to refine them through a finite-step Langevin chain \citep{xie2018cooperative}. The generative model is then trained to reproduce the refined samples from the corresponding latent variables. In each case, the effect of iterative sample refinement is progressively incorporated into the model that generated the initial samples.

RFPO adopts this refinement--learning structure for training a flow policy in online RL. The current flow policy generates actions, which are refined through finite-step MALA toward a target distribution induced by the Q-function. Each refined action is then paired with the noise used to generate the original action before refinement and serves as a self-generated target for FM. Whereas the preceding methods are developed for approximate inference or generative modeling, RFPO applies this principle to online policy learning, where the state-dependent target distribution is defined by a Q-function that is continually updated through environment interaction.

\begin{table*}[t]
    \centering
    \caption{
        Comparison of policy-training and action-refinement mechanisms among selected Q-guided generative policy methods. Refinement denotes explicit updates of action samples before they are used as training targets. FM and ULA denote flow matching and the unadjusted Langevin algorithm, respectively.}
    \label{tab:positioning}
    \setlength{\tabcolsep}{1.8pt}
    \begin{tabular}{
    @{}l
    p{0.40\textwidth}
    p{0.25\textwidth}
    l@{}
}
        \toprule
        Method
        & Policy training
        & \multicolumn{1}{c}{Refinement initialization}
        & \multicolumn{1}{c}{Refinement rule} \\
        \midrule

        QVPO
        & Q-weighted variational objective
        & \multicolumn{1}{c}{---}
        & \multicolumn{1}{c}{---} \\
        DPMD/SDAC
        & reweighted score matching
        & \multicolumn{1}{c}{---}
        & \multicolumn{1}{c}{---} \\
        QSM
        & Q-gradient score regression
        & \multicolumn{1}{c}{---}
        & \multicolumn{1}{c}{---} \\
        
        DACER
        & Q maximization via reverse diffusion
        & \multicolumn{1}{c}{---}
        & \multicolumn{1}{c}{---} \\
        RFM
        & posterior-mean velocity regression
        & \multicolumn{1}{c}{---}
        & \multicolumn{1}{c}{---} \\

        \midrule

        DIPO
        & diffusion fitting to refined actions
        & \multicolumn{1}{c}{replay-buffer actions}
        & \multicolumn{1}{c}{Q-gradient ascent} \\

        DACER-F
        & refined-action FM and Q maximization
        & \multicolumn{1}{c}{replay-buffer actions}
        & \multicolumn{1}{c}{ULA} \\

        \textbf{RFPO (ours)}
        & self-target FM
        & \multicolumn{1}{c}{current-policy actions}
        & \multicolumn{1}{c}{MALA} \\

        \bottomrule
    \end{tabular}
\end{table*}

\section{Preliminaries}
\label{sec:preliminaries}

\subsection{Problem Formulation}
\label{sec:problem_formulation}
We consider a continuous-action Markov decision process (MDP) defined by $\mathcal{M}=(\mathcal{S},\mathcal{A},P,r,\gamma)$, where $\mathcal{S}$ and $\mathcal{A}$ denote the state and action spaces, respectively. For a state $s\in\mathcal{S}$ and an action $a\in\mathcal{A}$, $P(s'\mid s,a)$ is the transition distribution over the next state $s'\in\mathcal{S}$, $r(s,a)$ is the reward function, and $\gamma\in[0,1)$ is the discount factor. Let $Q_\phi(s,a)$ denote a differentiable Q-function parameterized by $\phi$. Following the energy-based policy formulation commonly used in maximum-entropy reinforcement learning \citep{haarnoja2017reinforcement,haarnoja2018soft,levine2018reinforcement}, we adopt the Boltzmann distribution induced by the Q-function as the target action distribution:
\begin{equation}
    p_\alpha(a\mid s)
    =
    \frac{1}{Z_\alpha(s)}
    \exp\left(\frac{Q_\phi(s,a)}{\alpha}\right),
    \label{eq:target_action_distribution}
\end{equation}
where $\alpha>0$ is a temperature parameter and $Z_\alpha(s)=\int_{\mathcal{A}} \exp\left(Q_\phi(s,\bar a)/\alpha\right)d\bar a$ is the state-dependent normalizing constant. A smaller $\alpha$ concentrates probability mass more strongly around high-value actions, whereas a larger $\alpha$ spreads probability mass more broadly across the action space.

\subsection{Conditional Flow Matching}
\label{sec:flow_matching}
Flow matching (FM) learns a time-dependent vector field that transports samples from a simple base distribution such that their endpoint distribution matches a target distribution \citep{lipman2023flow}. Given a state $s$, let $x_0\sim p_0=\mathcal{N}(0,I_d)$ and $x_1\sim p_1(\cdot\mid s)$ denote samples from the base and conditional target distributions, respectively. A time-dependent vector field $v_\theta(s,x,t)$, with $t\in[0,1]$, defines the flow through
\begin{equation}
    \frac{d x_t}{dt}=v_\theta(s,x_t,t),
    \qquad x_{t=0}=x_0,
    \qquad
    x_1^\theta=\Phi_\theta(s,x_0),
    \label{eq:flow_ode}
\end{equation}
where $\Phi_\theta$ denotes the flow map obtained by integrating the ODE from $t=0$ to $t=1$. The objective is to learn a vector field such that the conditional distribution of $x_1^\theta$ matches $p_1(\cdot\mid s)$. To learn this vector field, conditional FM uses source--target pairs $(x_0,x_1)$. For $t\sim\mathcal{U}[0,1]$, we consider the linear interpolation and its associated velocity
\begin{equation}
    x_t=(1-t)x_0+t x_1,
    \qquad
    \frac{d x_t}{dt}=x_1-x_0.
    \label{eq:linear_probability_path}
\end{equation}
The resulting conditional FM objective is
\begin{equation}
    \mathcal{L}_{\mathrm{FM}}(\theta)
    =
    \mathbb{E}_{\substack{
        s\sim\rho,\,
        x_0\sim p_0,\,
        x_1\sim p_1(\cdot\mid s),\,
        t\sim\mathcal{U}[0,1]
    }}
    \left[
        \left\|
            v_\theta(s,x_t,t)-(x_1-x_0)
        \right\|_2^2
    \right],
    \label{eq:flow_matching_objective}
\end{equation}
where $\rho$ denotes the state distribution used for training. Under the standard FM formulation, minimizing \eqref{eq:flow_matching_objective} recovers the marginal vector field associated with the prescribed probability path, whose flow transports $p_0$ toward $p_1(\cdot\mid s)$. In the action-space setting considered below, we denote the flow-generated endpoint $x_1^\theta$ by $a_1$.

\subsection{Metropolis-Adjusted Langevin Algorithm}
\label{sec:mala}
The Metropolis-adjusted Langevin algorithm (MALA) is a Markov chain Monte Carlo method for sampling from a target distribution when direct sampling is difficult \citep{roberts1996exponential}. It uses gradient information to propose moves toward high-density regions and applies a Metropolis--Hastings correction to preserve the target distribution. Consider a target density $p_{\mathrm{tar}}(z)$ over $z\in\mathbb{R}^{d}$. Starting from an arbitrary initial point $z^{(0)}$ in the support of the target distribution, MALA constructs a sequence of iterates through successive updates. Given the current iterate $z^{(i)}$, MALA proposes a candidate $\widetilde{z}^{(i)}$:
\begin{equation}
    \widetilde{z}^{(i)}
    =
    z^{(i)}
    +\eta\nabla_z\log p_{\mathrm{tar}}(z^{(i)})
    +\sqrt{2\eta}\,\epsilon_i,
    \qquad
    \epsilon_i\sim\mathcal{N}(0,I_d),
    \label{eq:mala_proposal}
\end{equation}
where $\eta>0$ is the step size. The corresponding proposal density is
\begin{equation}
    q(z'\mid z)
    =
    \mathcal{N}\left(
        z';
        z+\eta\nabla_z\log p_{\mathrm{tar}}(z),
        2\eta I_d
    \right).
    \label{eq:mala_proposal_distribution}
\end{equation}
Here, $q(z'\mid z)$ denotes the probability density of proposing $z'$ from the current iterate $z$. In particular, the candidate $\widetilde{z}^{(i)}$ is drawn from $q(\cdot\mid z^{(i)})$. Because the proposal is generally asymmetric, MALA accepts the candidate with probability
\begin{equation}
    P_{\mathrm{acc}}\big(z^{(i)},\widetilde{z}^{(i)}\big)
    =
    \min\left\{
        1,\,
        \frac{
            p_{\mathrm{tar}}(\widetilde{z}^{(i)})\,
            q(z^{(i)}\mid\widetilde{z}^{(i)})
        }{
            p_{\mathrm{tar}}(z^{(i)})\,
            q(\widetilde{z}^{(i)}\mid z^{(i)})
        }
    \right\},
    \label{eq:mala_acceptance}
\end{equation}
where $P_{\mathrm{acc}}$ denotes the acceptance probability. Specifically, given $r_i\sim\mathcal{U}[0,1]$, the next iterate is determined as
\begin{equation}
    z^{(i+1)}
    =
    \begin{cases}
        \widetilde{z}^{(i)},
        & r_i\leq P_{\mathrm{acc}}(z^{(i)},\widetilde{z}^{(i)}),\\
        z^{(i)},
        & \text{otherwise}.
    \end{cases}
    \label{eq:mala_transition}
\end{equation}
The Metropolis--Hastings correction ensures that the resulting Markov transition leaves $p_{\mathrm{tar}}$ invariant. Moreover, because $p_{\mathrm{tar}}$ enters only through its log-density gradient and ratios of density values, its normalizing constant need not be evaluated.

\begin{figure}[tb!]
\begin{center}
\includegraphics[width=\linewidth]{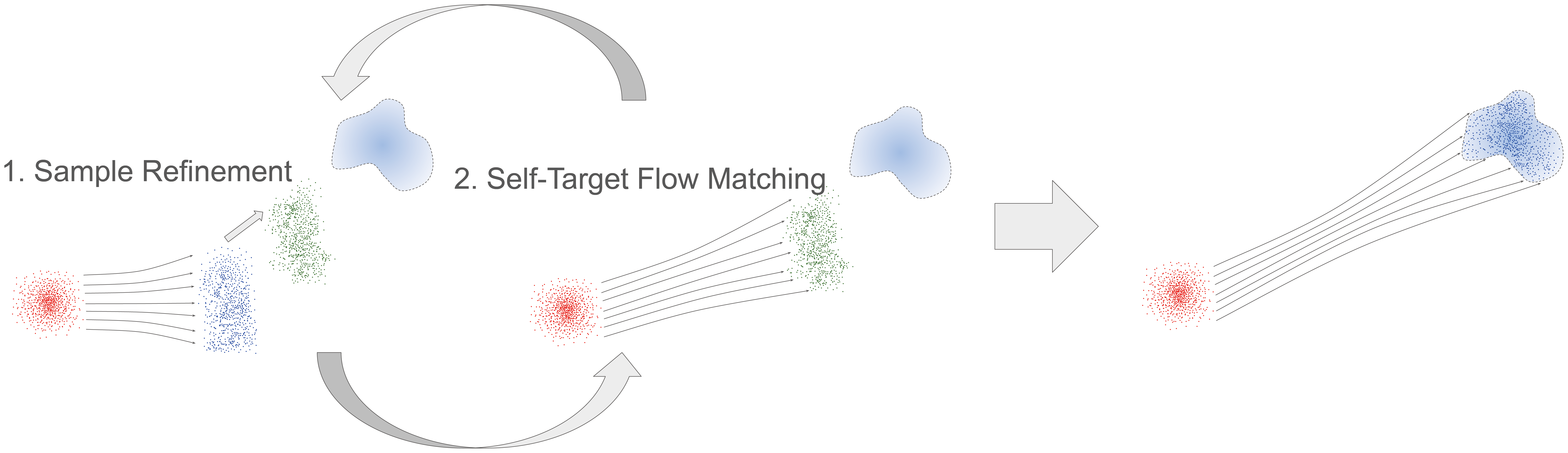}
\end{center}
\caption{Overview of RFPO. Shaded regions depict the target distribution $p_\alpha(\cdot\mid s)$; red, blue, and green dots denote Gaussian base samples $x_0$, flow-generated actions $a_1$, and refined actions $a_1'$, respectively. Q-guided sample refinement moves $a_1$ toward the target distribution (left), and self-target FM trains the policy on pairs of $a_1'$ and their corresponding base-noise samples $x_0$ (center). Curved arrows indicate repeated alternation between these processes, yielding a flow policy that approximates the target distribution (right).}
\label{fig:overall}
\end{figure}

\section{Refinement-Based Flow Policy Optimization}
\label{sec:method}
We propose Refinement-Based Flow Policy Optimization (RFPO), a novel framework for training a flow policy in online RL that combines Q-guided sample refinement with self-target FM, thereby removing the need for direct samples from the energy-based target distribution $p_\alpha(\cdot\mid s)$. RFPO alternates between two processes during training. First, the current flow policy generates actions from Gaussian noise. These actions are then refined toward $p_\alpha(\cdot\mid s)$ using gradient information from the Q-function. Second, the flow policy is trained through self-target FM using pairs of refined actions and their corresponding initial noise samples.

Rather than directly maximizing the Q-function through the flow parameters, RFPO uses Q-guided refinement to construct targets that are held fixed during the FM update. The resulting flow-policy update does not require backpropagating critic gradients through the iterative flow sampler and therefore avoids the associated sampler-dependent gradient path (see Appendix~\ref{app:one_step_interpretation} for an interpretation). Through this iterative cycle, RFPO is designed to shift the policy's probability mass toward high-value action regions while retaining its ability to represent multiple action modes. The two-process training procedure is illustrated in Figure~\ref{fig:overall}.

\subsection{Q-Guided Sample Refinement}
\label{sec:sample_refinement}
Given a state $s$, we first draw Gaussian noise $x_0\sim\mathcal{N}(0,I_d)$ and generate an action using the current flow policy:
\begin{equation}
    a_1=\Phi_\theta(s,x_0),
    \label{eq:flow_endpoint}
\end{equation}
where $\Phi_\theta$ denotes the flow map defined in Section~\ref{sec:flow_matching}. The action distribution induced by the current flow policy does not necessarily match the target distribution $p_\alpha(\cdot\mid s)$ defined in \eqref{eq:target_action_distribution}. We therefore refine $a_1$ toward $p_\alpha(\cdot\mid s)$ using MALA as follows. The MALA update in \eqref{eq:mala_proposal} uses the gradient of the log target density, which in our setting is $\nabla_a\log p_\alpha(a\mid s)=\frac{1}{\alpha}\nabla_a Q_\phi(s,a)$. 

Starting from $a^{(0)}=a_1$, we perform $K$ MALA steps using the current Q-function, with the action evolving through successive updates $\bigl(a^{(0)}\rightarrow a^{(1)}\rightarrow\cdots\rightarrow a^{(K)}\bigr)$. At iteration $i=0,\ldots,K-1$, we construct a candidate action $\widetilde{a}^{(i)}$ from $a^{(i)}$ as
\begin{equation}
    \widetilde{a}^{(i)}
    =
    a^{(i)}
    +
    \frac{\eta}{\alpha}
    \nabla_a Q_\phi(s,a^{(i)})
    +
    \sqrt{2\eta}\,\epsilon_i,
    \qquad
    \epsilon_i\sim\mathcal{N}(0,I_d),
    \label{eq:q_guided_proposal}
\end{equation}
where $\eta>0$ is the refinement step size. Substituting the target density $p_\alpha(\cdot\mid s)$ into \eqref{eq:mala_proposal_distribution} and \eqref{eq:mala_acceptance} yields the following:
\begin{align}
    &\text{Proposal density:}
    \notag\\
    &q_\phi(\widetilde{a}^{(i)}\mid a^{(i)},s)
    =
    \mathcal{N}\left(
        \widetilde{a}^{(i)};\,
        a^{(i)}+\frac{\eta}{\alpha}\nabla_a Q_\phi(s,a^{(i)}),
        2\eta I_d
    \right),
    \label{eq:q_guided_proposal_density}
    \\[4pt]
    &\text{Acceptance probability:}
    \notag\\
    &P_{\mathrm{acc}}\big(a^{(i)},\widetilde{a}^{(i)};s\big)
    =
    \min\left\{
        1,\,
        \exp\left(
            \frac{
                Q_\phi(s,\widetilde{a}^{(i)})
                -
                Q_\phi(s,a^{(i)})
            }{\alpha}
        \right)
        \frac{
            q_\phi(a^{(i)}\mid\widetilde{a}^{(i)},s)
        }{
            q_\phi(\widetilde{a}^{(i)}\mid a^{(i)},s)
        }
    \right\}.
    \label{eq:q_guided_acceptance}
\end{align}
The next action is set to $a^{(i+1)}=\widetilde{a}^{(i)}$ with probability $P_{\mathrm{acc}}(a^{(i)},\widetilde{a}^{(i)};s)$ and remains $a^{(i+1)}=a^{(i)}$ otherwise. After $K$ updates, we obtain the refined action $a_1'=a^{(K)}$.

We use MALA as a finite-step stochastic refinement procedure rather than running it until convergence. Accordingly, $a_1'$ is not assumed to be an exact sample from $p_\alpha(\cdot\mid s)$; the purpose of refinement is to bring the action distribution induced by the current flow policy closer to the target distribution. In the following subsection, we update the flow policy using $a_1'$ as a fixed target. Since the action spaces considered in our experiments are bounded, details on the action transformation and the corresponding MALA refinement are provided in Appendix~\ref{app:bounded_actions}.

\subsection{Self-Target Flow Matching}
\label{sec:self_target_flow_matching}
Once the refined action $a_1'$ is obtained, it serves as a self-generated target for training the flow policy. We pair it with the base noise $x_0$ that generated the original action $a_1=\Phi_\theta(s,x_0)$, forming the source--target pair $(x_0,a_1')$. Following the conditional FM formulation in Section~\ref{sec:flow_matching}, we sample $t\sim\mathcal{U}[0,1]$ and construct the linear interpolation $a_t=(1-t)x_0+t a_1'$. The flow policy is then trained using the self-target FM objective
\begin{equation}
    \mathcal{L}_{\mathrm{SFM}}(\theta)
    =
    \mathbb{E}_{\substack{
        s\sim\rho,\,
        x_0\sim\mathcal{N}(0,I_d),\\
        t\sim\mathcal{U}[0,1]
    }}
    \left[
        \left\|
            v_\theta(s,a_t,t)
            -
            (a_1'-x_0)
        \right\|_2^2
    \right],
    \label{eq:self_target_fm}
\end{equation}
where the expectation also accounts for the stochasticity of the finite-step MALA refinement used to construct $a_1'$. During this update, $a_1'$ is treated as a fixed target; gradients are propagated only through the velocity field $v_\theta$ and not through the flow sampling or MALA refinement used to generate $a_1'$.

\subsection{Critic Learning}
\label{sec:critic_learning}
We adopt a clipped double-Q critic update \citep{fujimoto2018addressing} using transitions sampled from a replay buffer $\mathcal{D}$. To construct the temporal-difference target, we generate an action from Gaussian noise using the online flow policy at the next state $s'$:
\begin{equation}
    x_0'\sim\mathcal{N}(0,I_d),
    \qquad
    a'=\Phi_{\theta}(s',x_0').
    \label{eq:target_next_action}
\end{equation}
The target uses the minimum of two target critics evaluated at $a'$. The generated action $a'$ is used directly in the temporal-difference target, without additional smoothing noise. Complete details of the critic update, target-critic updates, and the overall training procedure are provided in Appendix~\ref{app:training_details}.

\section{Theoretical Analysis}
\label{sec:theory}
We analyze whether self-target FM reproduces the refined action distribution and how alternating sample refinement and self-target FM changes the policy distribution relative to the target distribution. We consider an idealized population-level setting in which the self-target FM function class is realizable for the continuous conditional-mean velocity field $v^\star$ defined below, and each self-target FM update exactly recovers this field by minimizing $\mathcal{L}_{\mathrm{SFM}}$. We further assume that the flow ODE is solved without numerical integration error.

Let $\pi_\theta(\cdot\mid s)$ denote the policy distribution induced by $a=\Phi_\theta(s,x_0)$ with $x_0\sim\mathcal{N}(0,I_d)$. Let us define $\pi_n=\pi_{\theta_n}$ for the policy after $n$ cycles, where each cycle consists of MALA-based sample refinement followed by self-target FM. The Q-function, temperature $\alpha$, and refinement step size $\eta$ are all held fixed across the cycles under analysis.

For a fixed state $s$ and policy parameter $\theta_n$, let $\mu_n(\cdot\mid s)$ denote the distribution of the refined action $a_1'=a^{(K)}$ after $K$ MALA transitions. Let $p_t(\cdot\mid s)$ denote the distribution of the linearly interpolated action $a_t=(1-t)x_0+t a_1'$ defined in Section~\ref{sec:self_target_flow_matching}. We define the velocity field $v^\star$ as the conditional mean of the target velocity $a_1'-x_0$ as follows:
\begin{equation}
    v^\star(s,x,t)
    :=
    \mathbb{E}\bigl[
        a_1'-x_0 \,\big|\, a_t=x,\,s
    \bigr].
    \label{eq:theory_optimal_velocity}
\end{equation}
Under the stated idealizations and Assumption~\ref{ass:reg} in Appendix~\ref{app:assumptions}, the following results hold for $\rho$-almost every state $s$, where $\rho$ is the state distribution used for self-target FM.

\begin{lemma}[Self-target FM reproduces the refined action distribution]
\label{lem:coupling}
The velocity field $v^\star$ minimizes the state-wise self-target FM objective. The action generated by integrating $v^\star$ from Gaussian noise follows the refined action distribution:
\[
    x_0\sim\mathcal{N}(0,I_d)
    \quad\Longrightarrow\quad
    \Phi^{v^\star}(s,x_0)\sim\mu_n(\cdot\mid s),
\]
where $\Phi^{v^\star}(s,\cdot)$ denotes the flow map induced by $v^\star$.
\end{lemma}

Lemma~\ref{lem:coupling} establishes that self-target FM reproduces the refined action distribution even when each refined action is paired with its original source noise. This is an application of the general conditional FM framework for source--target couplings \citep[Lemma~3.1]{pooladian2023multisample} \citep[Theorems~3.1 and 3.2]{tong2024improving}. The pairing affects the transport path, while the terminal distribution remains the refined action distribution $\mu_n(\cdot\mid s)$.

Let $M_s$ denote one MALA transition targeting $p_\alpha(\cdot\mid s)$, and let $\pi M_s^K$ denote the distribution obtained by drawing an action from $\pi$ and applying $K$ transitions. The next lemma characterizes the combined effect of sample refinement and self-target FM.

\begin{lemma}[One cycle is equivalent to $K$ MALA transitions]
\label{lem:chain}
Each cycle satisfies
\[
    \pi_{n+1}(\cdot\mid s)
    =
    \mu_n(\cdot\mid s)
    =
    \bigl(\pi_nM_s^K\bigr)(\cdot\mid s).
\]
Consequently,
\[
    \pi_n(\cdot\mid s)
    =
    \bigl(\pi_0M_s^{nK}\bigr)(\cdot\mid s).
\]
\end{lemma}

Lemma~\ref{lem:chain} shows that finite-step refinement accumulates across policy updates. Since the MALA kernel leaves the target distribution invariant, the data-processing inequality yields the following guarantee.

\begin{theorem}[Monotone decrease of KL divergence to the target]
\label{thm:improve}
Every cycle satisfies
\[
    \mathrm{KL}\bigl(
        \pi_{n+1}(\cdot\mid s)\,\|\,p_\alpha(\cdot\mid s)
    \bigr)
    \leq
    \mathrm{KL}\bigl(
        \pi_n(\cdot\mid s)\,\|\,p_\alpha(\cdot\mid s)
    \bigr).
\]
Equivalently, for the fixed critic $Q_\phi$, the entropy-regularized policy objective
\[
    \mathcal{J}(\pi\mid s)
    :=
    \mathbb{E}_{a\sim\pi(\cdot\mid s)}[Q_\phi(s,a)]
    +
    \alpha\mathcal{H}(\pi(\cdot\mid s))
\]
satisfies $\mathcal{J}(\pi_{n+1}\mid s)\geq\mathcal{J}(\pi_n\mid s)$, with both values finite. These conclusions hold for every fixed step size $\eta>0$ and every number of refinement steps $K\geq1$. Equality holds whenever $\pi_n$ is invariant for $M_s^K$. The target distribution $p_\alpha(\cdot\mid s)$ is a fixed point of the policy update.
\end{theorem}

These results show that self-target FM preserves the distributional effect of refinement, allowing finite-step MALA refinement to accumulate across policy updates and monotonically improve the entropy-regularized policy objective for the fixed critic. This does not require MALA to converge within each cycle. Theorem~\ref{thm:improve} alone does not establish convergence to the target. If $M_s$ additionally satisfies conditions ensuring convergence of the MALA chain to $p_\alpha(\cdot\mid s)$ \citep{roberts1996exponential}, Lemma~\ref{lem:chain} implies that the policy distribution also converges to the target. All proofs and the extension to bounded actions through the coordinate transformation are provided in Appendix~\ref{app:proof}.

\section{Experiments}
\label{sec:experiments}
% We first examine whether alternating Q-guided sample refinement and self-target FM can capture complex multimodal structure across six synthetic two-dimensional target distributions with diverse geometries without mode collapse. We then evaluate RFPO on continuous-control benchmarks.

\begin{figure*}[t]
    \centering
    \begin{subfigure}[t]{0.32\textwidth}
        \centering
        \includegraphics[width=\linewidth]
        {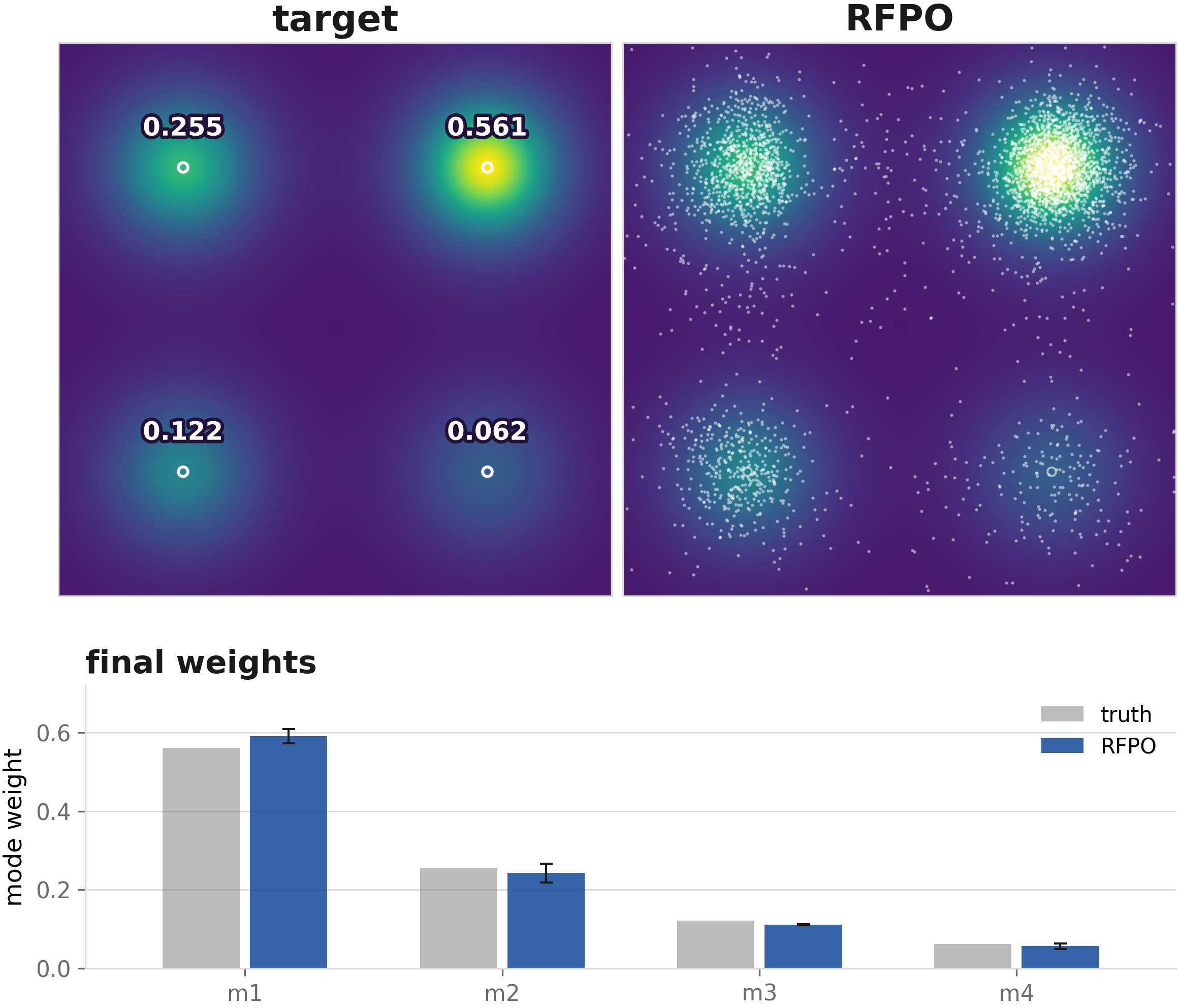}
        \caption{Isotropic mixture.}
        \label{fig:toy_2d_compact}
    \end{subfigure}
    \hfill
    \begin{subfigure}[t]{0.32\textwidth}
        \centering
        \includegraphics[width=\linewidth]
        {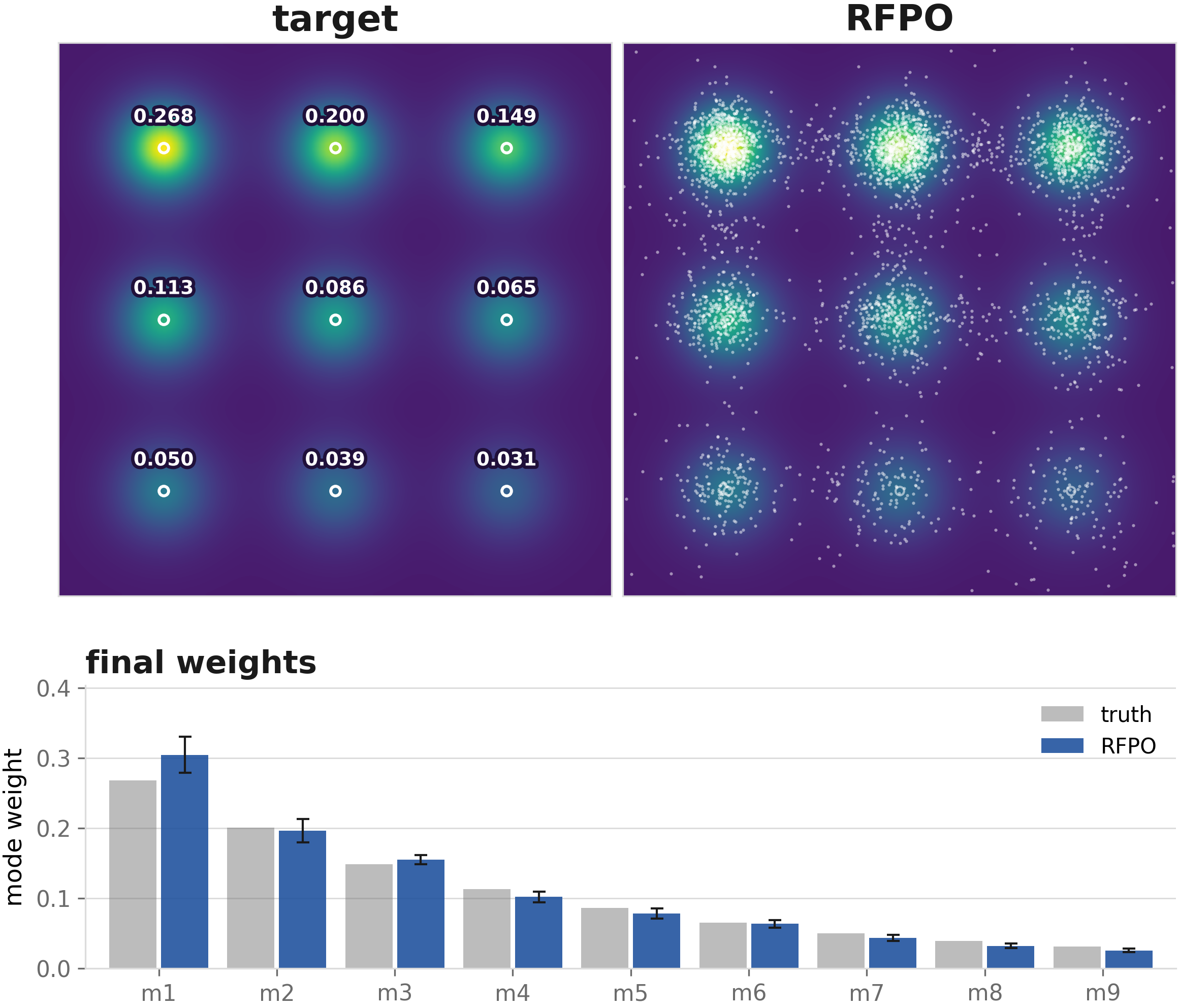}
        \caption{Grid mixture.}
        \label{fig:toy_grid_compact}
    \end{subfigure}
    \hfill
    \begin{subfigure}[t]{0.32\textwidth}
        \centering
        \includegraphics[width=\linewidth]
        {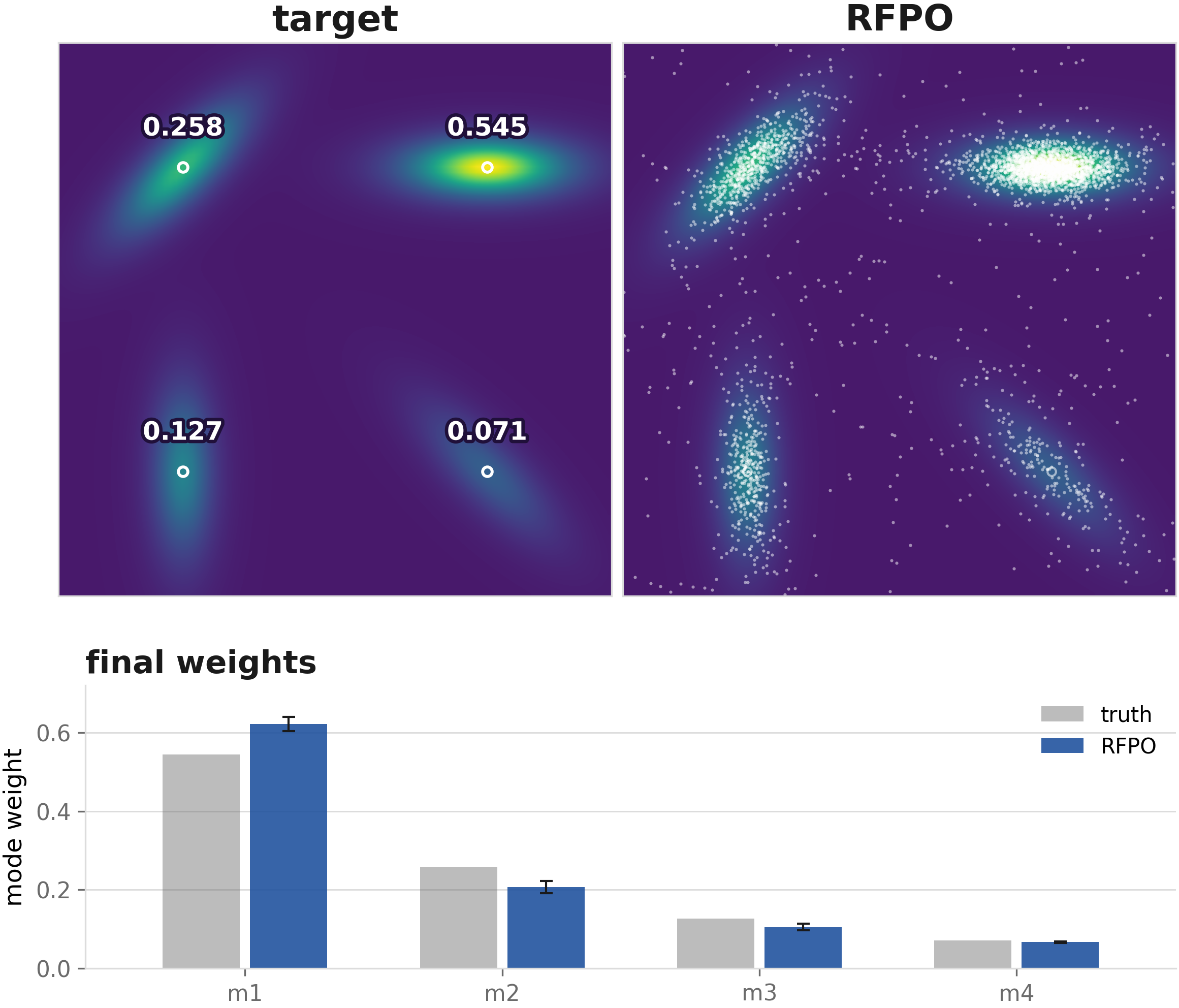}
        \caption{Anisotropic mixture.}
        \label{fig:toy_aniso_compact}
    \end{subfigure}
    \medskip
    \begin{subfigure}[t]{0.32\textwidth}
        \centering
        \includegraphics[width=\linewidth]
        {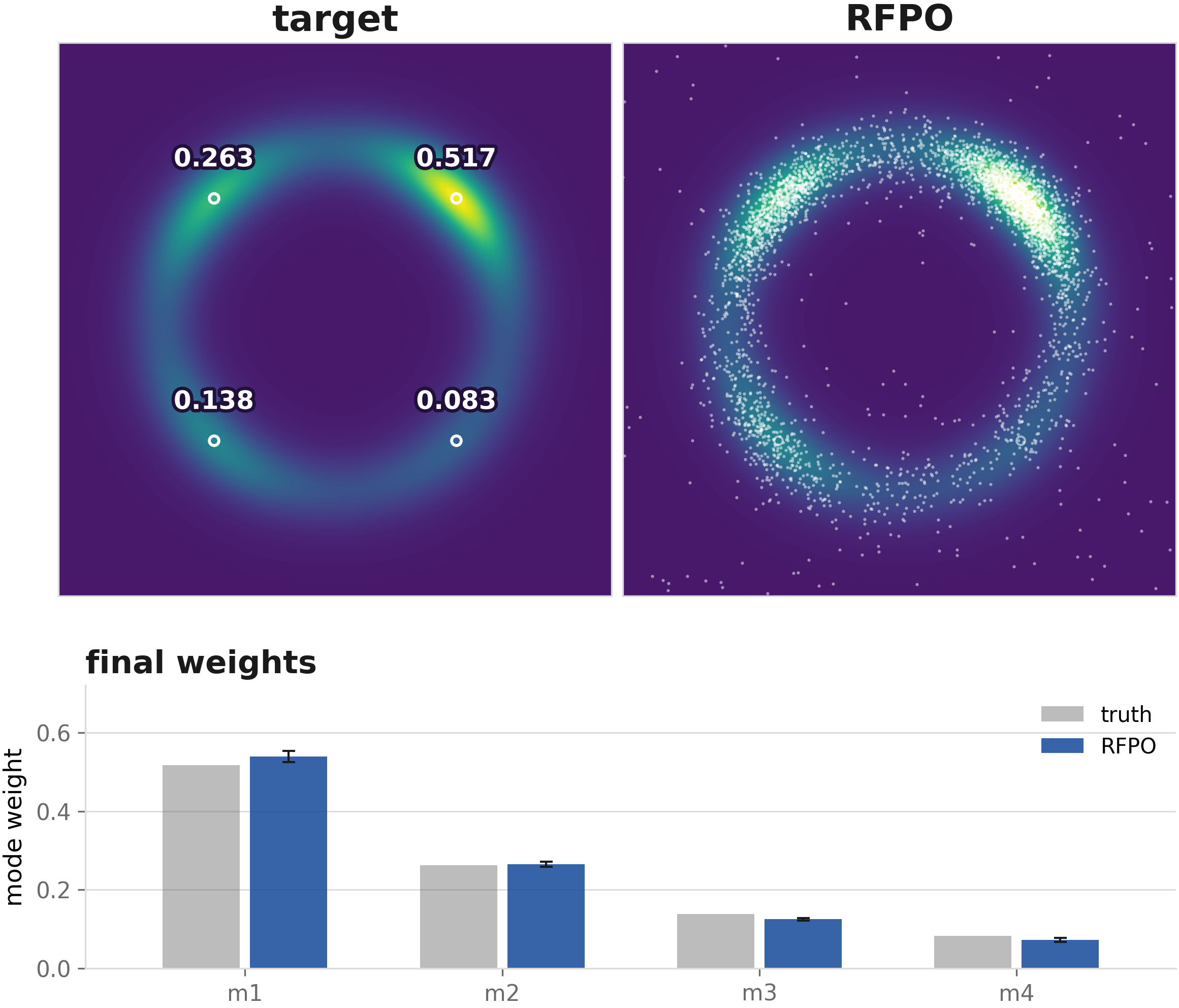}
        \caption{Ring-shaped mixture.}
        \label{fig:toy_ring_compact}
    \end{subfigure}
    \hfill
    \begin{subfigure}[t]{0.32\textwidth}
        \centering
        \includegraphics[width=\linewidth]
        {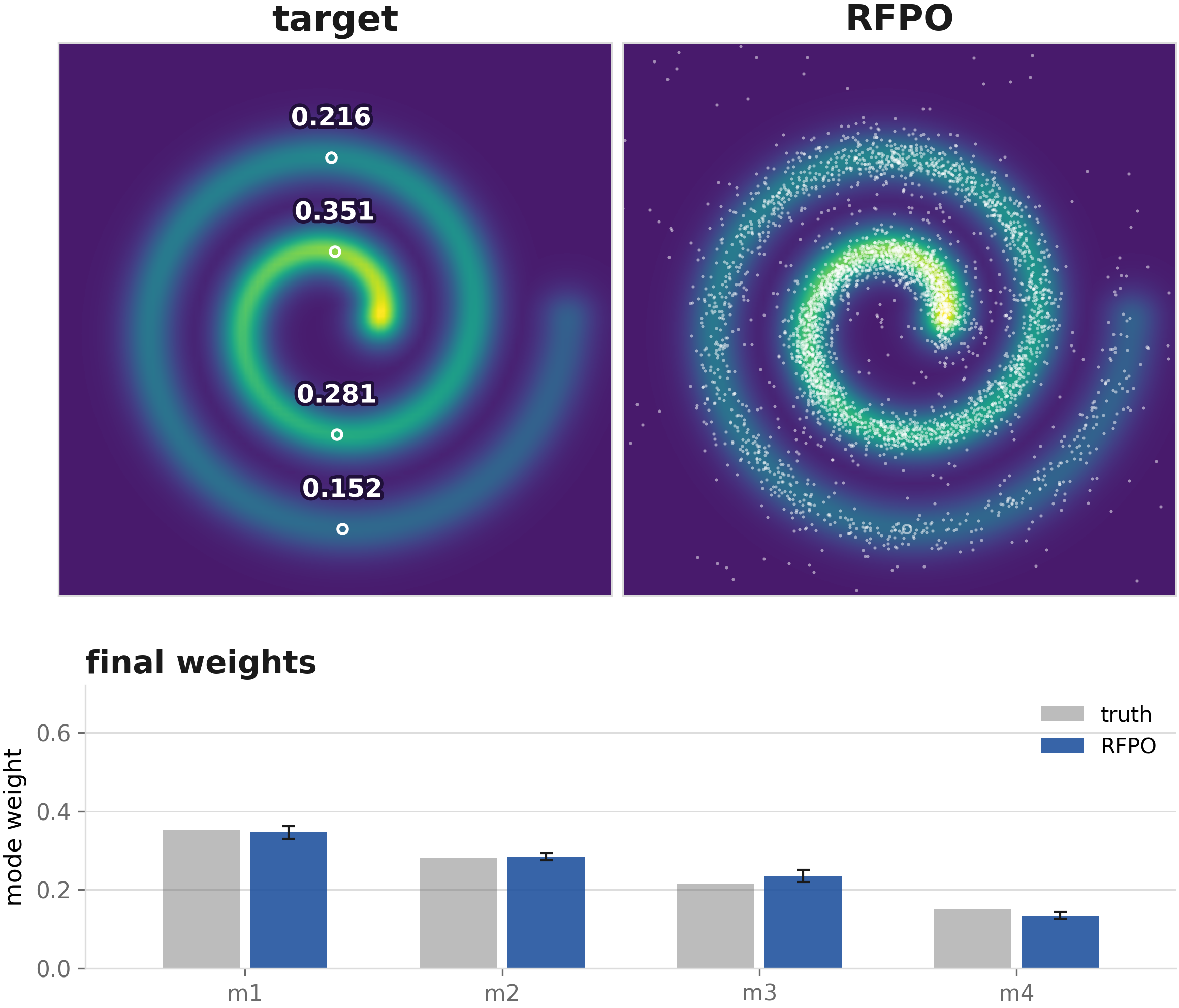}
        \caption{Spiral-shaped mixture.}
        \label{fig:toy_spiral_compact}
    \end{subfigure}
    \hfill
    \begin{subfigure}[t]{0.32\textwidth}
        \centering
        \includegraphics[width=\linewidth]
        {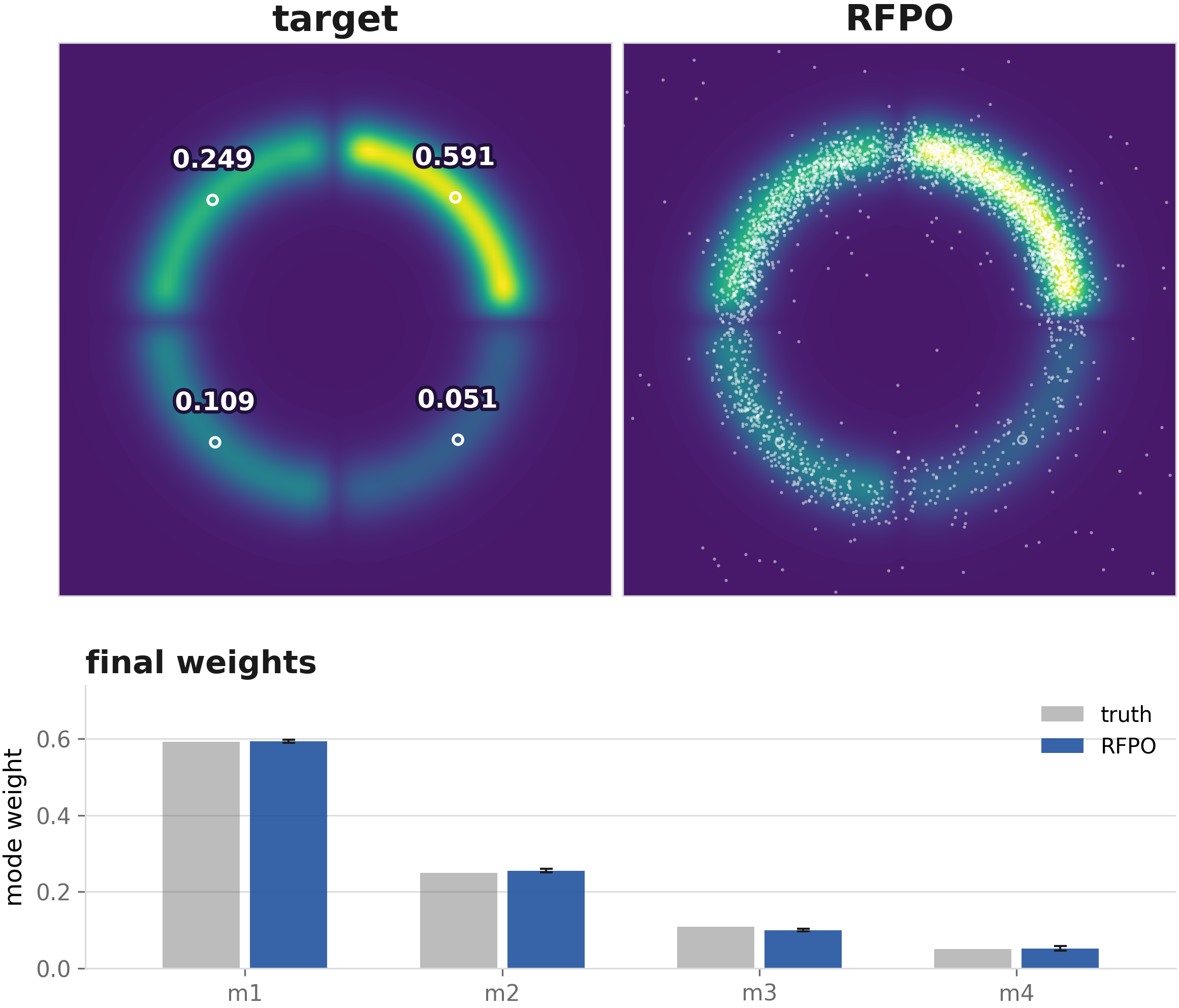}
        \caption{Arc-shaped mixture.}
        \label{fig:toy_arc_compact}
    \end{subfigure}
    \caption{In each subfigure, the upper-left panel shows the target density, with white markers indicating predefined components and annotations giving their ground-truth component-region masses. The upper-right panel overlays learned-policy samples on the target density. The lower panel compares ground-truth component-region masses (gray) with sample estimates (blue); error bars indicate the standard deviation across four random seeds.}
    \label{fig:toy_examples_compact}
\end{figure*}

\subsection{Learning Multimodal Distributions with Diverse Geometries}
\label{sec:multimodal_distributions}
We examine whether alternating Q-guided sample refinement and self-target FM can learn multimodal target distributions without collapsing onto a single high-value region. We consider the six two-dimensional target geometries described in Appendix~\ref{app:shapes}, including separated Gaussian-shaped components, anisotropic components, and nonlinear ridge structures with unequal prescribed component heights. Our method accesses each target only through its $Q$-function and learns from the self-generated refined actions without using samples drawn directly from the target distribution. We perform 3,000 cycles, each consisting of sample refinement followed by self-target FM. The hyperparameters for these experiments are provided in Table~\ref{tab:hyperparameters_toy} in Appendix~\ref{app:hyperparameters}.

As shown in Figure~\ref{fig:toy_examples_compact}, the learned policy generates samples across multiple high-density regions and captures the overall geometry of each target, including spatially separated modes and extended structures such as rings and spirals. These results demonstrate that alternating Q-guided sample refinement and self-target FM captures complex multimodal structure across the evaluated targets without collapsing onto a single high-value region. The full results, including the evolution of the policy-generated action distributions and component-region masses over the 3,000 cycles, are provided in Appendix~\ref{app:toy_full_results}.

\subsection{Performance on Continuous-Control Benchmarks}
\label{sec:continuous_control}
We evaluate RFPO on six continuous-control environments from MuJoCo \citep{todorov2012mujoco}. The baselines include diffusion-based online RL methods DPMD and SDAC \citep{ma2025efficient}, as well as DIPO \citep{yang2023policy}, which refines replay-buffer actions through Q-gradient ascent before training a diffusion policy. We also include RFM \citep{li2026reverse} as a Q-guided flow-policy baseline and SAC \citep{haarnoja2018soft} as a representative Gaussian-policy baseline. All methods are trained for 1M environment steps. Evaluation uses common environment seeds and episode counts while following each method's standard action-generation convention; full details are provided in Appendix~\ref{app:training_details}. The hyperparameters used for RFPO are provided in Table~\ref{tab:hyperparameters}. Further details on the baselines are provided in Appendix~\ref{app:baselines}.
%남은 baseline 우선순위:DACER, QVPO, QSM

The learning curves for RFPO and the baselines are shown in Figure~\ref{fig:learning_curves}, and the final returns in Table~\ref{tab:return}. RFPO matches or outperforms representative online RL baselines in almost every environment. These results demonstrate that the proposed Q-guided training procedure effectively trains a flow policy in online RL without requiring direct samples from the energy-based target distribution. 
% Together with the preceding synthetic experiments, these findings support the applicability of RFPO's refinement and self-target FM procedure to both fixed target distributions and online RL with a learned Q-function.

\begin{table}[t]
\centering
\caption{Final smoothed return on six MuJoCo tasks, averaged over five seeds.}
\label{tab:return}
\footnotesize
\setlength{\tabcolsep}{2.5pt}
\begin{tabular}{lrrrrrr}
\toprule
Method & HalfCheetah & Walker2d & Hopper & Ant & Swimmer & \shortstack[r]{Humanoid\\Standup} \\
\midrule
\textbf{RFPO (ours)} & \textbf{13{,}269 $\pm$ 516} & \textbf{5{,}456 $\pm$ 262} & 2{,}110 $\pm$ 122 & \underline{5{,}854 $\pm$ 501} & \textbf{343 $\pm$ 5} & \textbf{158.4 $\pm$ 5.1} \\
SAC & 10{,}228 $\pm$ 942 & 4{,}710 $\pm$ 379 & \underline{2{,}969 $\pm$ 683} & 4{,}061 $\pm$ 587 & \underline{340 $\pm$ 7} & 149.5 $\pm$ 13.0 \\
SDAC & 12{,}663 $\pm$ 796 & 3{,}535 $\pm$ 654 & 2{,}026 $\pm$ 737 & $-$41 $\pm$ 78 & 291 $\pm$ 128 & \underline{153.7 $\pm$ 2.6} \\
DPMD & 11{,}274 $\pm$ 848 & 4{,}377 $\pm$ 304 & \textbf{3{,}038 $\pm$ 292} & 5{,}746 $\pm$ 350 & 190 $\pm$ 134 & 125.1 $\pm$ 14.6 \\
DIPO & \underline{12{,}928 $\pm$ 265} & \underline{4{,}939 $\pm$ 441} & 2{,}466 $\pm$ 684 & \textbf{6{,}033 $\pm$ 208} & 295 $\pm$ 106 & 147.2 $\pm$ 16.9 \\
RFM & 9{,}705 $\pm$ 764 & 3{,}742 $\pm$ 316 & 1{,}583 $\pm$ 219 & $-$7 $\pm$ 1 & 277 $\pm$ 125 & 147.7 $\pm$ 16.1 \\
\bottomrule
\end{tabular}

{\footnotesize \textbf{Bold}: best. \underline{Underline}: second best. HumanoidStandup returns are in units of $10^3$.}
\end{table}

\begin{figure*}[t]
    \centering
    \includegraphics[width=0.9\textwidth]{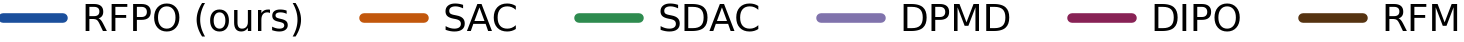}
    \par\smallskip

    % Row 1
    \begin{subfigure}[t]{0.32\textwidth}
        \centering
        \includegraphics[width=\linewidth]
        {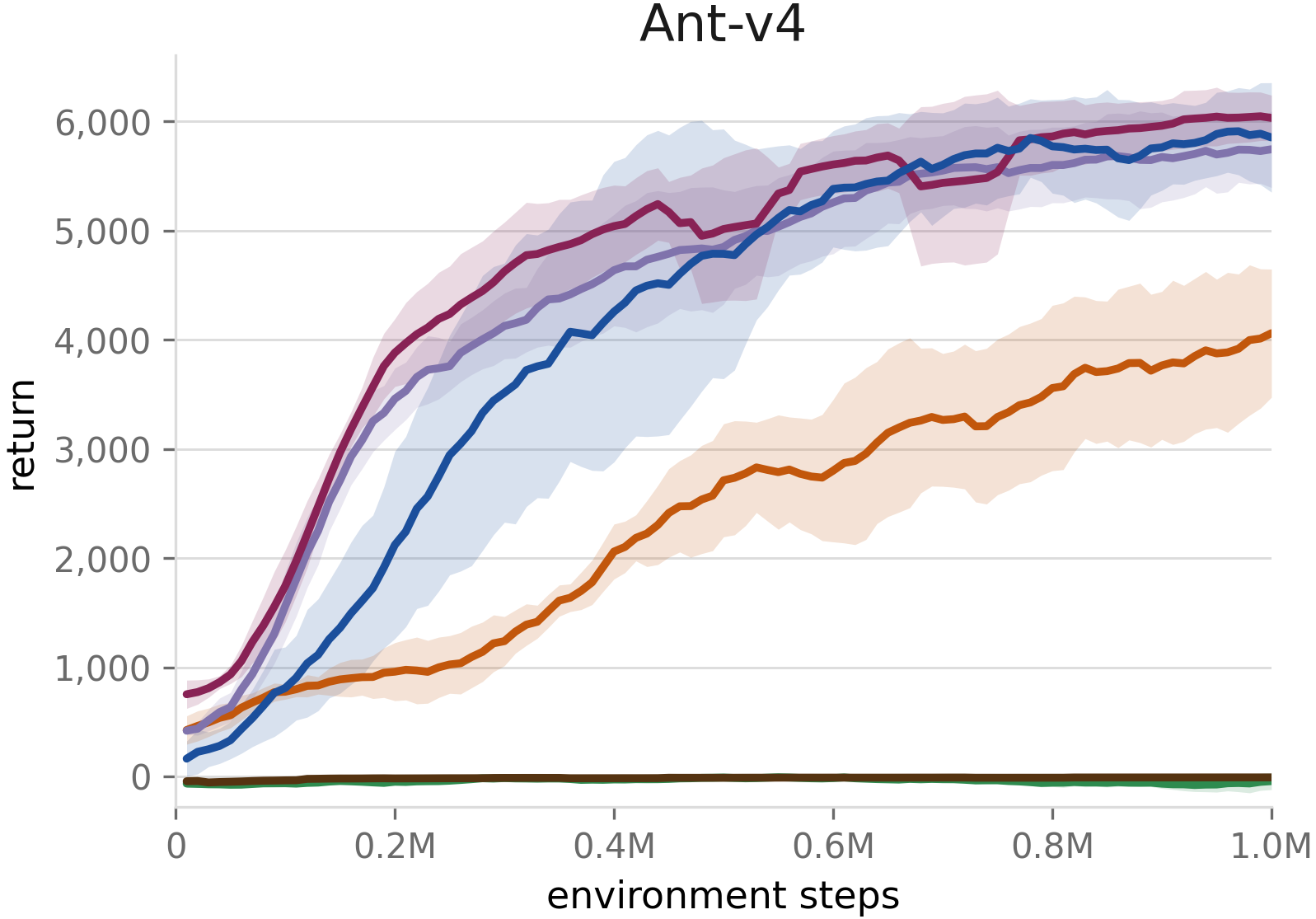}
    \end{subfigure}
    \hfill
    \begin{subfigure}[t]{0.32\textwidth}
        \centering
        \includegraphics[width=\linewidth]
        {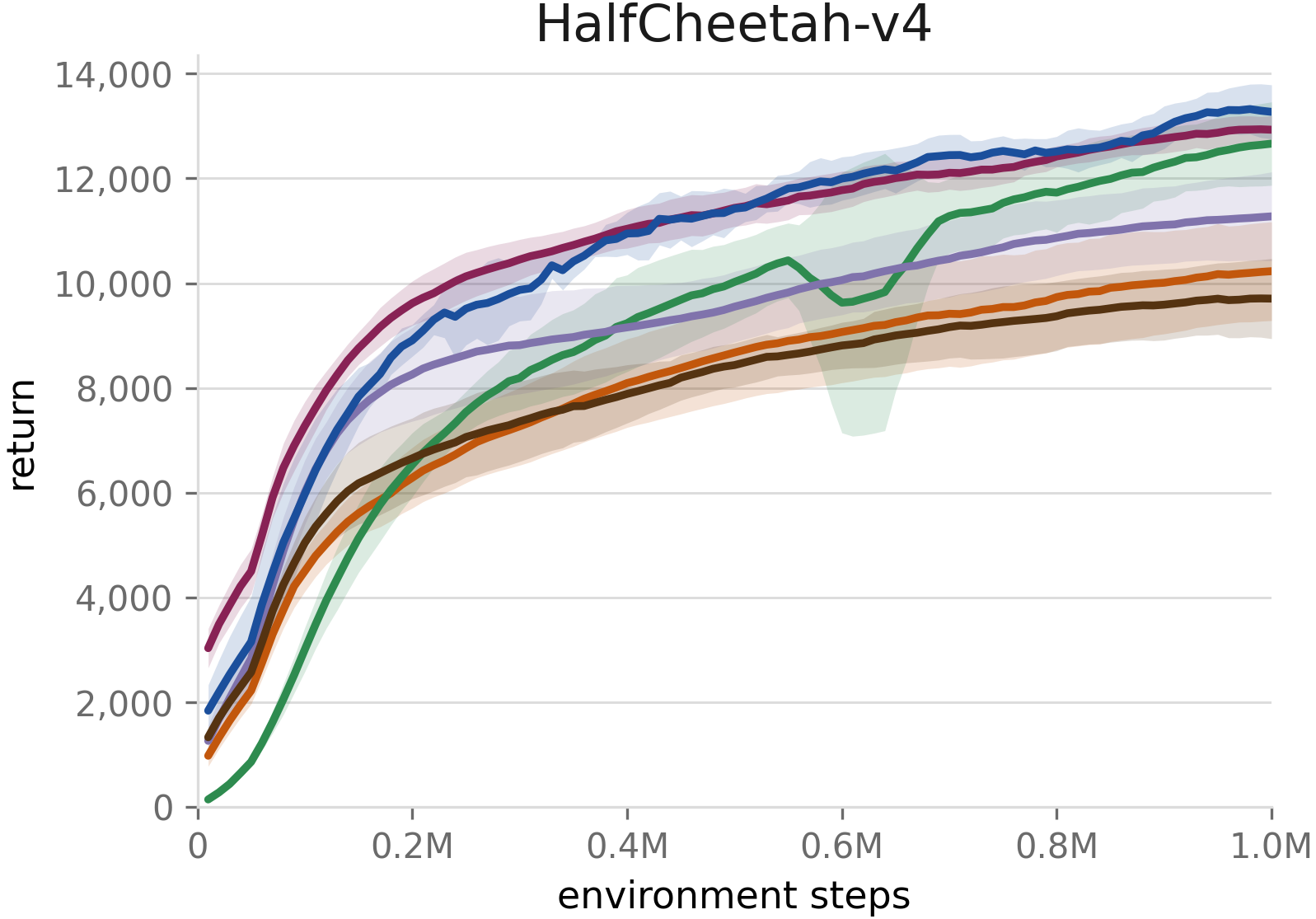}
    \end{subfigure}
    \hfill
    \begin{subfigure}[t]{0.32\textwidth}
        \centering
        \includegraphics[width=\linewidth]
        {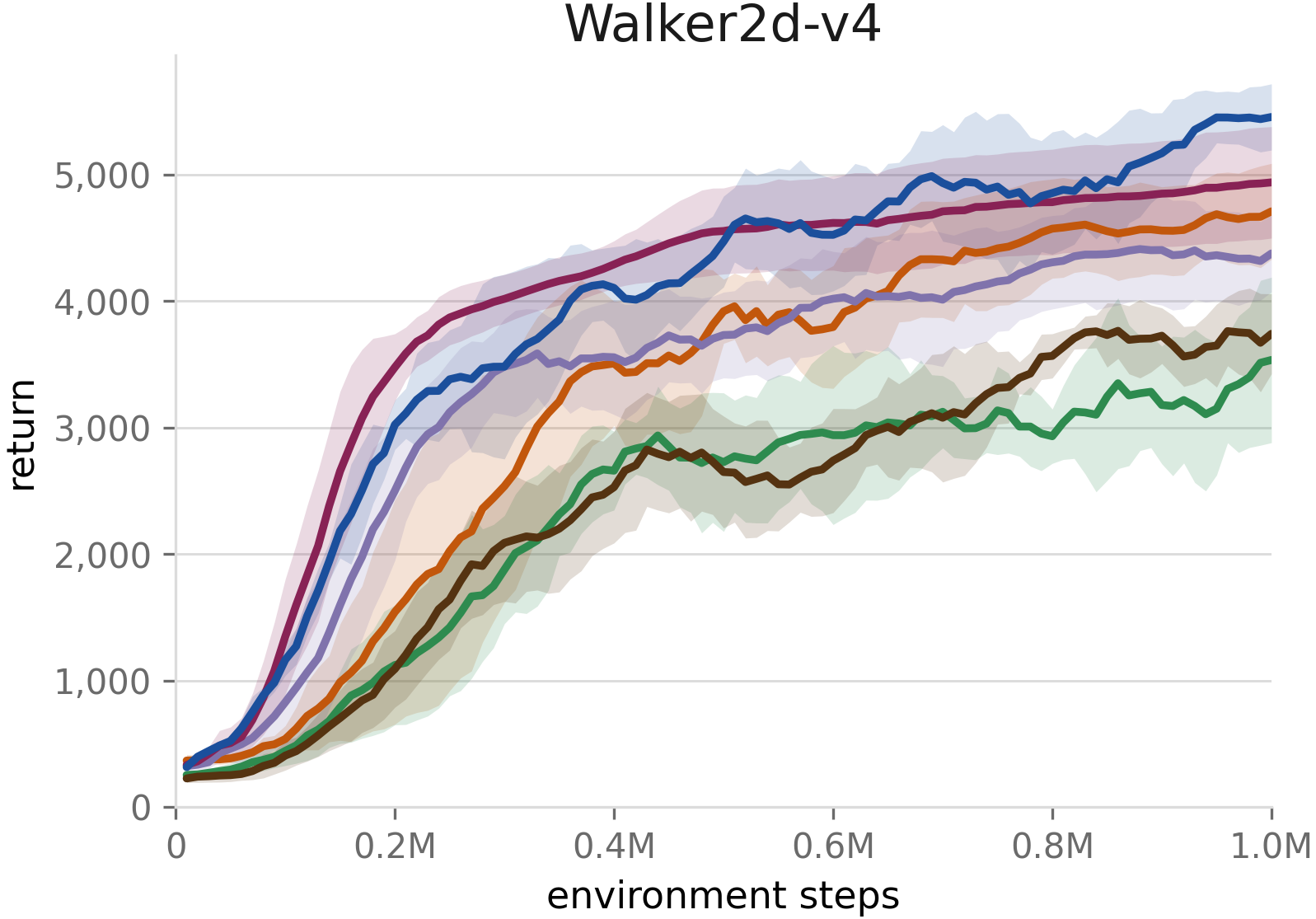}
    \end{subfigure}

    \medskip

    % Row 2
    \begin{subfigure}[t]{0.32\textwidth}
        \centering
        \includegraphics[width=\linewidth]
        {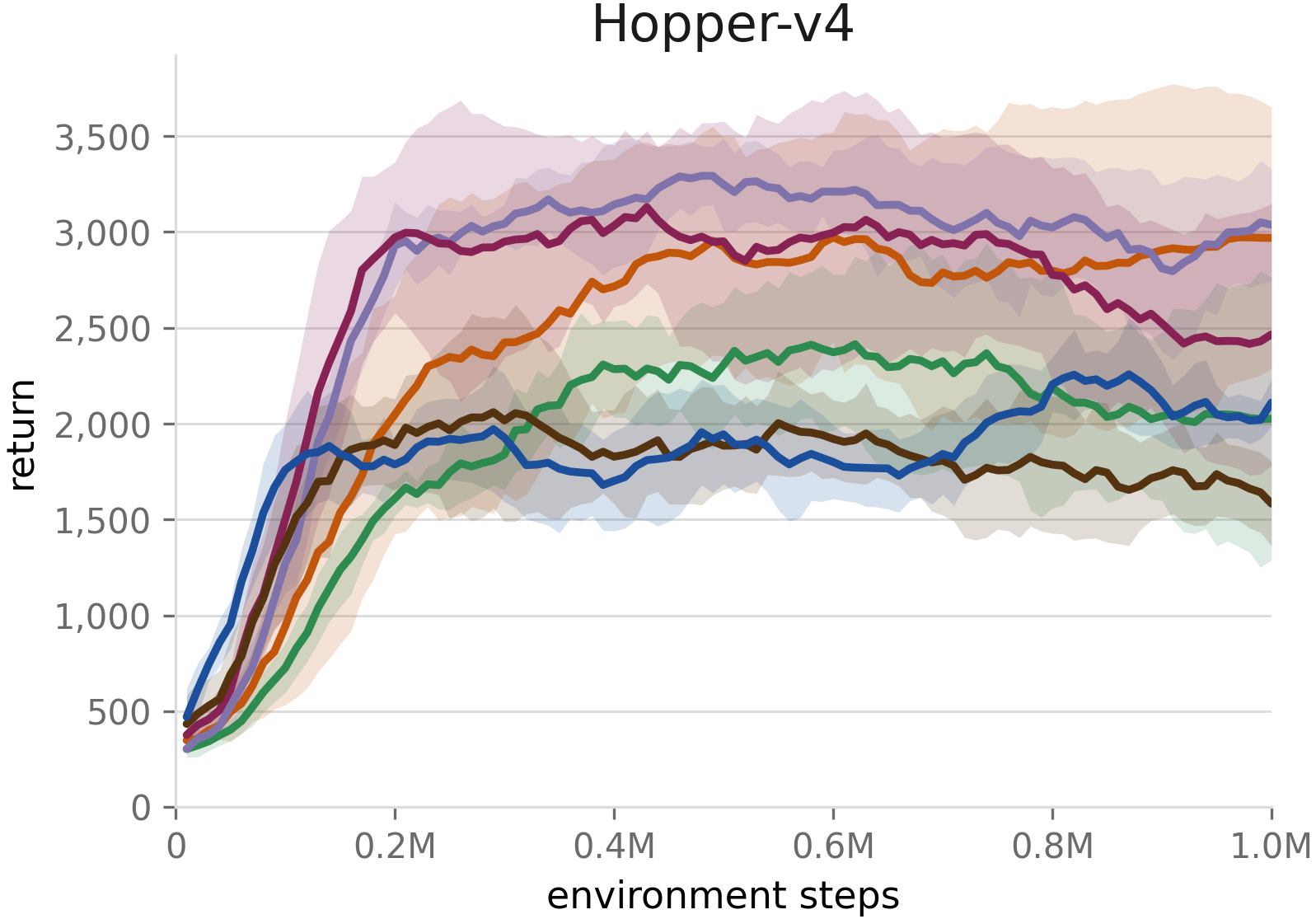}
    \end{subfigure}
    \hfill
    \begin{subfigure}[t]{0.32\textwidth}
        \centering
        \includegraphics[width=\linewidth]
        {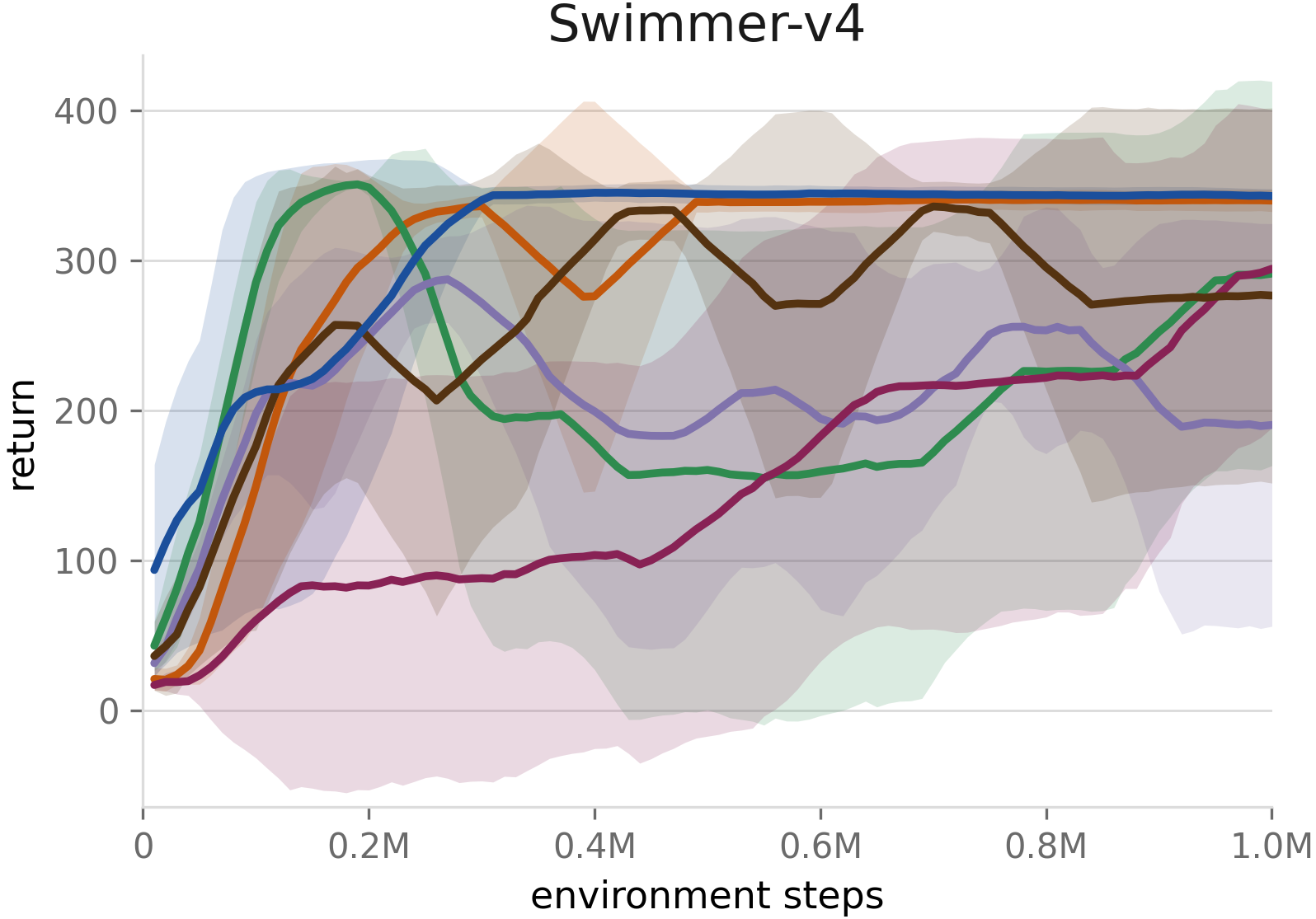}
    \end{subfigure}
    \hfill
    \begin{subfigure}[t]{0.32\textwidth}
        \centering
        \includegraphics[width=\linewidth]
        {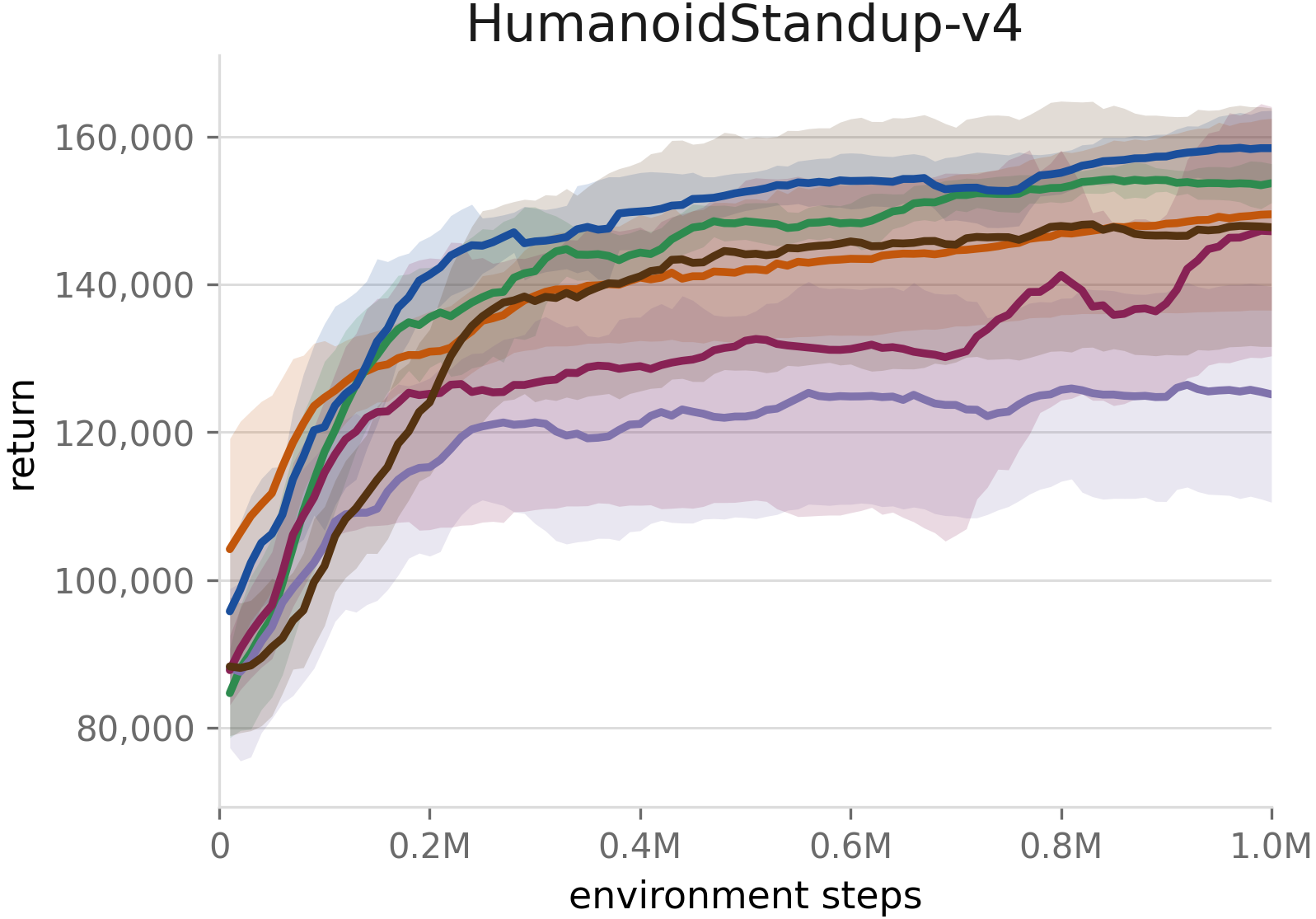}
    \end{subfigure}

    \caption{Learning curves on six continuous-control tasks. Each curve represents the average return over 5 random seeds, with the shaded area indicating one standard deviation from the mean.}
    \label{fig:learning_curves}
\end{figure*}

\section{Conclusion}
We introduce RFPO, a novel framework for training an expressive flow policy in online RL without requiring direct samples from the energy-based target distribution induced by the Q-function. RFPO uses the Q-function to refine actions generated by the current flow policy and uses the refined actions as self-generated targets for self-target FM. Through this iterative cycle, the policy distribution progressively moves toward high-value action regions while retaining its ability to represent multiple action modes.

We further provide a theoretical analysis of the distributional dynamics induced by RFPO. Under exact self-target FM optimization and exact flow integration, each cycle of sample refinement and self-target FM corresponds to applying $K$ MALA transitions to the current policy distribution, establishing the energy-based target distribution as a fixed point of the iterative procedure. Experiments on six synthetic two-dimensional target distributions with diverse geometries demonstrate that alternating Q-guided sample refinement and self-target FM captures complex multimodal structure without collapsing onto a single mode. Across six continuous-control environments in MuJoCo, RFPO further matches or outperforms representative online RL baselines in almost every environment.

\paragraph{Limitations.}
% While RFPO shows promising results on MuJoCo benchmarks and captures multimodal target distributions, several limitations remain to be addressed. First, its iterative updates introduce additional training cost because constructing each self-generated target requires a finite number of MALA transitions. Nevertheless, the computational analysis in Appendix~\ref{app:compute} indicates that RFPO requires substantially less network computation per update than most of evaluated generative-policy baselines.

First, our theoretical analysis considers an idealized setting in which the critic is held fixed, the self-target FM problem is solved exactly, and the flow ODE is integrated without numerical error. In practical online training, the critic continually evolves and the policy is updated using finite optimization steps. These assumptions isolate the distributional effect of the proposed refinement--learning cycle and provide an exact reference for understanding the algorithm. Extending the analysis to account for approximate optimization and a moving target distribution remains an important direction for future work.

Second, RFPO includes several design choices, including the refinement kernel, the acceptance correction, and initialization from current-policy samples. The current experiments do not fully isolate the contribution of each component. Nevertheless, the central idea of RFPO is to learn a policy toward the target distribution by iteratively alternating between MCMC-based sample refinement and self-target FM. A systematic ablation of these design choices and their combinations may further improve the performance and efficiency of RFPO and remains an important direction for future work.

\section*{AI Use Statement}
We used generative AI tools for brainstorming, refining the conceptual framing, and obtaining feedback on research methodology and experimental design. We also used these tools to discuss theoretical arguments, assist with proof development and checking, search for and summarize relevant literature, and assist with method implementation. For manuscript preparation, we used generative AI tools to assist with translation, organization, drafting, and editing for clarity and readability. We did not use generative AI tools to generate synthetic datasets, clean or reformat datasets, propose or refine research hypotheses, formulate the mathematical claims presented in this work, or interpret experimental results. AI-assisted proof checking was used as a supplement to the authors' own verification. The authors reviewed and revised the AI-assisted material and take responsibility for the final manuscript, implementation, theoretical claims, proofs, and conclusions.

\subsection*{Reproducibility statement}
We provide implementation and experimental details sufficient to reproduce the proposed method and evaluate its claims. The complete RFPO procedure, including critic learning, target-critic updates, termination handling, stop-gradient treatment of refined targets, and the order of refinement and flow-matching updates, is specified in Appendix~\ref{app:training_details}. The bounded-action implementation, including the tanh transformation, induced target density, and MALA score, is detailed in Appendix~\ref{app:bounded_actions}. All hyperparameters for the synthetic and MuJoCo experiments are reported in Tables~\ref{tab:hyperparameters_toy} and \ref{tab:hyperparameters}, respectively. The evaluation protocol, including the evaluation frequency, number and seeds of evaluation episodes, action-generation conventions, and learning-curve smoothing procedure, is provided in Appendix~\ref{app:training_details}. The six synthetic target distributions and the component assignments used for region-mass evaluation are specified in Appendix~\ref{app:shapes}, and the full synthetic results across training cycles are reported in Appendix~\ref{app:toy_full_results}. Baseline configurations are described in Appendix~\ref{app:baselines}. Finally, the assumptions and complete proofs for the theoretical results in Section~\ref{sec:theory} are provided in Appendix~\ref{app:proof}.

% \subsection*{Reproducibility statement}

% (This section is \textbf{recommended} and does not count toward the page limit.)

% It is important that the work published in ICLR is reproducible. Authors are
% strongly encouraged to include a paragraph-long Reproducibility Statement at the
% end of the main text (before references) to discuss the efforts that have been
% made to ensure reproducibility. This paragraph should not itself describe
% details needed for reproducing the results, but rather reference the parts of
% the main paper, appendix, and supplemental materials that will help with
% reproducibility. For example, for novel models or algorithms, a link to an
% anonymous downloadable source code can be submitted as supplementary materials;
% for theoretical results, clear explanations of any assumptions and a complete
% proof of the claims can be included in the appendix; for any datasets used in
% the experiments, a complete description of the data processing steps can be
% provided in the supplementary materials. Each of the above are examples of
% things that can be referenced in the reproducibility statement.

\bibliography{iclr2027_conference}
\bibliographystyle{iclr2027_conference}

\clearpage
\appendix

\section{Comparison with Direct Q Maximization}
\label{app:one_step_interpretation}

A straightforward approach to training a flow policy toward the energy-based target distribution induced by the Q-function is to augment the flow-matching objective with a Q-maximization term. For $\lambda_Q>0$, this objective takes the form
\begin{equation}
    \mathcal{L}_{\mathrm{SFM+Q}}(\theta)
    =
    \mathcal{L}_{\mathrm{SFM}}(\theta)
    -
    \lambda_Q
    \mathbb{E}_{s,x_0}
    \left[
        Q_\phi\bigl(s,\Phi_\theta(s,x_0)\bigr)
    \right].
    \nonumber
\end{equation}
For a fixed $(s,x_0)$ pair, the gradient contribution of the Q-maximization term is
\begin{equation}
    -\lambda_Q
    \left(
        \frac{\partial\Phi_\theta(s,x_0)}{\partial\theta}
    \right)^\top
    \nabla_a Q_\phi\bigl(s,\Phi_\theta(s,x_0)\bigr).
    \nonumber
\end{equation}
Thus, directly maximizing the Q-function requires backpropagating the critic gradient through the flow sampler. For an iterative flow sampler, the flow-sampler sensitivity $\partial\Phi_\theta(s,x_0)/\partial\theta$ contains parameter-sensitivity terms from each ODE-solver step, multiplied by products of state Jacobians. Consequently, the scale of the critic gradient reaching the flow parameters can be attenuated or amplified by the learned velocity field along the sampled trajectory.

Figure~\ref{fig:amplification} illustrates these sampler-dependent state-Jacobian factors on flows trained with the procedure of Section~\ref{sec:multimodal_distributions}, run on the isotropic mixture for $400$ cycles at five temperatures with one seed each. Let $x_0,\ldots,x_N$ denote an Euler trajectory with step size $\Delta t$, and define
\begin{equation}
    J_j
    =
    I+
    \Delta t\,
    \frac{\partial v_\theta(s,x_j,t_j)}{\partial x_j},
    \qquad
    A_i
    =
    \left\|
        J_{N-1}\cdots J_i
    \right\|_2.
    \nonumber
\end{equation}
Here, $A_i$ measures how the flow dynamics scale a perturbation entering at solver step $i$. Products of this form appear in the flow-sampler sensitivity in the direct Q-maximization gradient.

As shown in Figure~\ref{fig:amplification}(a), the sensitivity varies substantially along a single trajectory. In particular, perturbations entering early can be strongly attenuated, whereas later perturbations can remain nearly unchanged or be mildly amplified. Figure~\ref{fig:amplification}(b) shows that this scale also varies across temperatures when all other settings are held fixed. Across the five temperatures, $A_i$ ranges from $0.08$ to $1.31$. These measurements are not gradient-norm or optimization-stability measurements. Rather, they illustrate the sampler-dependent Jacobian factors through which a direct Q-maximization update must propagate critic gradients.

\begin{figure}[h]
    \centering
    \includegraphics[width=\textwidth]{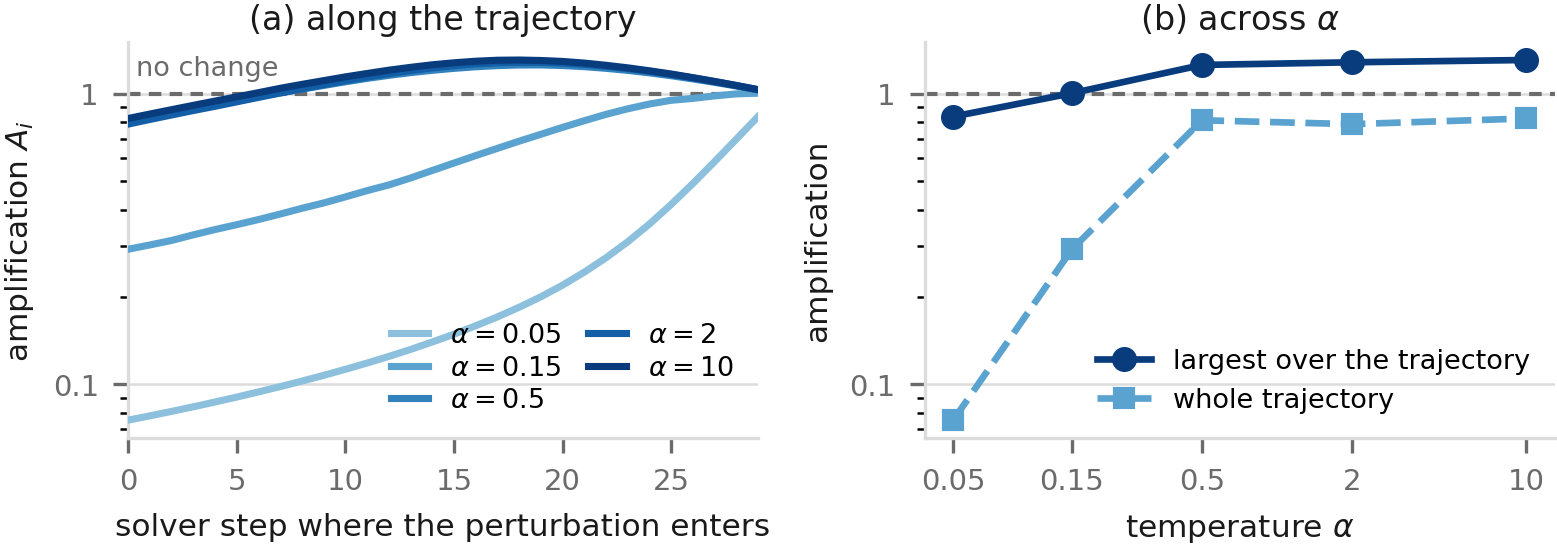}
    \caption{State-Jacobian sensitivity factors along flows trained with the procedure of Section~\ref{sec:multimodal_distributions} for $400$ cycles, one seed per temperature. (a)~$A_i$ measures the scaling of a perturbation entering at solver step $i$; the dashed line indicates no change in norm. Contributions entering early can be strongly attenuated, whereas later contributions remain near their original scale or are mildly amplified. (b)~The largest value of $A_i$ along a trajectory and the whole-trajectory factor $A_0$ are shown across temperatures. These quantities illustrate a sampler-sensitivity component of the direct Q-maximization gradient, rather than an observed optimization failure.}
    \label{fig:amplification}
\end{figure}

In contrast, RFPO uses the Q-function to construct fixed targets for self-target FM. To illustrate this distinction, we consider a simplified deterministic one-step refinement. We omit the Gaussian proposal noise and assume that the proposed action is accepted. This simplification is used only for interpretation and is not the refinement procedure used by RFPO. Under this simplification, one refinement step gives
\begin{equation}
    a_1'
    =
    a_1
    +
    \frac{\eta}{\alpha}
    \nabla_a Q_\phi(s,a_1).
    \nonumber
\end{equation}
The target velocity used by self-target FM is therefore
\begin{equation}
    a_1'-x_0
    =
    (a_1-x_0)
    +
    \frac{\eta}{\alpha}
    \nabla_a Q_\phi(s,a_1).
    \nonumber
\end{equation}
Thus, refinement incorporates the Q-gradient by shifting the velocity-regression target. For
\begin{equation}
    a_t'
    =
    (1-t)x_0+t a_1',
    \nonumber
\end{equation}
the per-pair self-target FM loss is
\begin{equation}
    \ell_{\mathrm{SFM}}
    =
    \left\|
        v_\theta(s,a_t',t)
        -
        \left[
            (a_1-x_0)
            +
            \frac{\eta}{\alpha}
            \nabla_a Q_\phi(s,a_1)
        \right]
    \right\|_2^2.
    \nonumber
\end{equation}
During the self-target FM update, $a_1$, $a_1'$, $a_t'$, and $\nabla_a Q_\phi(s,a_1)$ are treated as fixed with respect to $\theta$. The flow-parameter gradient is consequently
\begin{equation}
    \nabla_\theta \ell_{\mathrm{SFM}}
    =
    2
    \left(
        \frac{\partial v_\theta(s,a_t',t)}{\partial\theta}
    \right)^\top
    \left[
        v_\theta(s,a_t',t)
        -
        \left(
            (a_1-x_0)
            +
            \frac{\eta}{\alpha}
            \nabla_a Q_\phi(s,a_1)
        \right)
    \right].
    \nonumber
\end{equation}
The Q-gradient therefore appears only as part of the fixed regression target. In particular, the update contains no derivative through $\Phi_\theta(s,x_0)$ and no critic gradient backpropagated through the flow sampling procedure. Thus, RFPO uses Q-gradient information to modify the velocity target rather than directly optimizing the Q-function through the flow parameters, avoiding this sampler-dependent gradient path. This comparison does not imply that direct Q maximization necessarily diverges; rather, it identifies a gradient path that RFPO is designed to avoid.

\section{Bounded Action Implementation}
\label{app:bounded_actions}

The main text describes RFPO directly in the action space. In the bounded-action implementation, both the flow and MALA refinement operate on a variable $u\in\mathbb{R}^{d}$, which is mapped to an environment action through a tanh transformation. Let $a_{\min},a_{\max}\in\mathbb{R}^{d}$ denote the componentwise action bounds, with $a_{\min,j}<a_{\max,j}$, and define the action half-range $\Delta a=(a_{\max}-a_{\min})/2$. The transformation is
\begin{equation}
    T(u)
    =
    \frac{a_{\max}+a_{\min}}{2}
    +
    \Delta a\odot\tanh(u),
    \label{eq:bounded_action_transform}
\end{equation}
where $\tanh$ is applied elementwise and $\odot$ denotes elementwise multiplication. The flow generates an endpoint $u_1=\Phi_\theta(s,x_0)$ from $x_0\sim\mathcal{N}(0,I_d)$, and the corresponding action is $a_1=T(u_1)$. Thus, the flow map produces $u_1$ in this implementation, while $T$ enforces the action bounds. For numerical stability, we clip $u$ componentwise to $[-10,10]$ after each Euler integration step. This clipping is applied only during flow sampling, not during MALA refinement; refined endpoints may therefore lie outside this box.

\paragraph{Target density and gradient.}
To perform refinement in $u$-space while targeting $p_\alpha(\cdot\mid s)$ in action space, we use the change-of-variables formula
\begin{equation}
    p_\alpha^u(u\mid s)=p_\alpha(T(u)\mid s)\left|\det J_T(u)\right|,
    \label{eq:target_action_distribution_bound}
\end{equation}
where $J_T(u)$ is the Jacobian of $T$. Substituting \eqref{eq:target_action_distribution} gives
\begin{equation}
    p_\alpha^u(u\mid s)
    \propto
    \exp\left(
        \frac{Q_\phi(s,T(u))}{\alpha}
    \right)
    \prod_{j=1}^{d}
    \Delta a_j
    \left(1-\tanh^2(u_j)\right).
    \label{eq:explicit_bounded_target}
\end{equation}
The gradient of the log target density is
\begin{equation}
\begin{aligned}
    g_\phi^u(s,u)
    &:=
    \nabla_u\log p_\alpha^u(u\mid s) \\
    &=
    \frac{1}{\alpha}
    \left[
        \left.\nabla_a Q_\phi(s,a)\right|_{a=T(u)}
        \odot\Delta a
        \odot\left(1-\tanh^2(u)\right)
    \right]
    -2\tanh(u).
\end{aligned}
    \label{eq:bounded_target_score}
\end{equation}
The first term follows from differentiating the Q-function through $T$, and the second term comes from the log-Jacobian determinant. The normalizing constant and the constant Jacobian factor $\prod_j\Delta a_j$ vanish from the log-density gradient. The action half-range $\Delta a$ remains in the first term through the derivative of $T$.

\paragraph{Refinement and policy learning.}
Starting from $u^{(0)}=u_1$, we perform $K$ MALA steps targeting $p_\alpha^u(\cdot\mid s)$. At iteration $i=0,\ldots,K-1$, the candidate $\widetilde{u}^{(i)}$ is constructed from $u^{(i)}$ using $g_\phi^u(s,u^{(i)})$ with the norm clipping described in Appendix~\ref{app:hyperparameters}. Both forward and reverse proposal densities are evaluated using this clipped gradient, while the target-density terms in the Metropolis--Hastings acceptance ratio use $p_\alpha^u(\cdot\mid s)$ without modification. The candidate is then accepted or rejected to determine $u^{(i+1)}$. Step-size adaptation follows \eqref{eq:eta_adaptation}.

After $K$ MALA steps, we obtain the refined endpoint $u_1'=u^{(K)}$, which is paired with its original source noise $x_0$ and treated as a fixed target. Self-target FM is performed in the same $u$-space: for $t\sim\mathcal{U}[0,1]$, we construct the linear interpolation $u_t=(1-t)x_0+t u_1'$ and regress $v_\theta(s,u_t,t)$ onto $u_1'-x_0$. This is the objective in \eqref{eq:self_target_fm} with $(a_t,a_1')$ replaced by $(u_t,u_1')$. The corresponding refined action is $T(u_1')$.

\section{Critic Learning and Training Procedure}
\label{app:training_details}

Let
\[
    A_\theta(s,x)
    :=
    T\bigl(\Phi_\theta(s,x)\bigr)
\]
denote the bounded action generated by the flow policy from base noise
$x$, where $T$ is the action transformation in
Appendix~\ref{app:bounded_actions}. We use
$\operatorname{sg}[\cdot]$ to denote stop-gradient.

\paragraph{Critic update.}
Let $Q_{\phi_1}$ and $Q_{\phi_2}$ denote the online critics, and let
$Q_{\bar{\phi}_1}$ and $Q_{\bar{\phi}_2}$ denote the corresponding
target critics. For a replay-buffer transition $(s,a,r,s',d)$, we
draw $x_0'\sim\mathcal{N}(0,I_d)$ and define the temporal-difference
target as
\begin{equation}
    y
    =
    r
    +
    \gamma(1-d)
    \min_{i\in\{1,2\}}
    Q_{\bar{\phi}_i}\bigl(s',A_\theta(s',x_0')\bigr).
    \label{eq:critic_target}
\end{equation}
The next action in \eqref{eq:critic_target} is generated by the
online flow policy; RFPO does not maintain a target flow policy. The
target contains neither an entropy term nor additional target-policy
smoothing noise.

The critics are optimized using
\begin{equation}
    \mathcal{L}_{Q}
    =
    \sum_{i=1}^{2}
    \mathbb{E}_{(s,a,r,s',d)\sim\mathcal{D}}
    \left[
        \bigl(Q_{\phi_i}(s,a)-y\bigr)^2
    \right].
    \label{eq:critic_loss}
\end{equation}
The loss is implemented as a mean over the batch and a sum over the
two critics. Critic gradients are clipped to a global norm of $10$.

The replay buffer stores $d=1$ only for environment terminations.
Time-limit truncations reset the environment but retain the bootstrap
term in \eqref{eq:critic_target}.

\paragraph{Critic used for refinement and target update.}
MALA refinement uses the online clipped-double critic
\begin{equation}
    Q_{\phi}^{\min}(s,a)
    :=
    \min\bigl\{
        Q_{\phi_1}(s,a),Q_{\phi_2}(s,a)
    \bigr\}.
    \label{eq:refinement_critic}
\end{equation}
Specifically, in \eqref{eq:explicit_bounded_target}, the Q-function is
given by $Q_{\phi}^{\min}$. The same target density is used to compute
both the proposal score and the Metropolis--Hastings acceptance
probability. Gradients through the minimum in
\eqref{eq:refinement_critic} are computed by automatic
differentiation.

After each critic and flow-policy update, the target critics are
updated by Polyak averaging:
\begin{equation}
    \bar{\phi}_i
    \leftarrow
    (1-\tau)\bar{\phi}_i+\tau\phi_i,
    \qquad
    i\in\{1,2\},
    \qquad
    \tau=0.005.
    \label{eq:target_critic_update}
\end{equation}

\paragraph{Training procedure.}
\begin{algorithm}[t]
\caption{RFPO training procedure}
\label{alg:rfpo}
\begin{algorithmic}[1]
\State Initialize $\theta$, online critics $\phi_1,\phi_2$, target
critics $\bar{\phi}_1,\bar{\phi}_2$, and replay buffer $\mathcal{D}$.
\State Receive initial environment state $s$.
\For{each environment step}
    \State Select $a$ uniformly during warm-up; otherwise sample
    $x_0\sim\mathcal{N}(0,I_d)$ and set $a=A_\theta(s,x_0)$ without gradients.
    \State Execute $a$, observe $(r,s',\mathrm{terminated},\mathrm{truncated})$,
    and store $(s,a,r,s',d)$ in $\mathcal{D}$, where
    $d=\mathbf{1}\{\mathrm{terminated}\}$.

    \If{after warm-up}
        \State Sample $\{(s_b,a_b,r_b,s_b',d_b)\}_{b=1}^{B}$ from $\mathcal{D}$.
        \State Sample $x_{0,b}'\sim\mathcal{N}(0,I_d)$ and compute
        $\displaystyle y_b=r_b+\gamma(1-d_b)
        \min_i Q_{\bar{\phi}_i}(s_b',A_\theta(s_b',x_{0,b}'))$
        without gradients.
        \State Update $\phi_1,\phi_2$ using \eqref{eq:critic_loss}.

        \State Sample $u_{0,b}\sim\mathcal{N}(0,I_d)$ and obtain
        $\displaystyle
        u_{1,b}'\gets
        \operatorname{sg}\!\left[
            \mathrm{MALA}_K\bigl(
                \Phi_\theta(s_b,u_{0,b});
                p_\alpha^u(\cdot\mid s_b)
            \bigr)
        \right]$
        \For{$g=1,\ldots,G$}
            \State Sample $t_b\sim\mathcal{U}[0,1]$ and set
            $u_{t,b}=(1-t_b)u_{0,b}+t_bu_{1,b}'$.
            \State Update $\theta$ using
            $\displaystyle
            \frac{1}{B}\sum_{b=1}^{B}
            \left\|
                v_\theta(s_b,u_{t,b},t_b)-(u_{1,b}'-u_{0,b})
            \right\|_2^2$.
        \EndFor
        \State Update $\bar{\phi}_i\gets(1-\tau)\bar{\phi}_i+\tau\phi_i$
        for $i\in\{1,2\}$.
    \EndIf
    \State Reset if terminated or truncated; otherwise set $s\gets s'$.
\EndFor
\end{algorithmic}
\end{algorithm}

In the reported experiments, warm-up uses 5,000 uniformly sampled actions. Thereafter, each environment step performs one critic update and $G=5$ flow-matching updates. The flow-generated and refined batch is constructed once per training iteration and reused across the $G$ flow-matching updates; only the interpolation time $t$ is resampled. The flow endpoint, MALA refinement, and refined targets are generated. Thus, gradients of the self-target FM objective propagate only through $v_\theta(s,u_t,t)$.

\paragraph{Evaluation protocol.}
Every $10{,}000$ environment steps the current policy is evaluated on five episodes with environment seeds $1000$ through $1004$, and the mean return over those five episodes is recorded; each episode runs until termination or truncation. RFPO is evaluated with actions sampled from the flow policy, which has no deterministic mode: zeroing its source noise does not produce one. Each baseline is evaluated under the convention used in its own paper, which for every baseline considered here is a deterministic action, and those are the numbers reported in Figure~\ref{fig:learning_curves} and Table~\ref{tab:return}. The comparison is in this respect not symmetric, and the asymmetry works against RFPO. Warm-up also differs: RFPO and SAC collect $5{,}000$ uniformly sampled actions before training begins and the diffusion- and flow-policy baselines collect $10{,}000$, so the first evaluation falls at $10{,}000$ environment steps for the former and at $20{,}000$ for the latter.

\paragraph{Smoothing of reported returns.}
Each seed's evaluation series is placed on a common grid at $10{,}000$-step intervals and smoothed with a nine-point centred moving average, a window of $90{,}000$ environment steps. The mean and standard deviation across the five seeds are then taken over the smoothed series rather than over the raw evaluations. The order matters: a five-episode evaluation is noisy enough that averaging first would let a single seed's jitter set the edge of the shaded band. At the two ends of a run the window shrinks to the evaluations that exist rather than padding with zeros, so the curves extend to $1$M steps at both ends. The value reported in Table~\ref{tab:return} is the final point of this smoothed curve.

\section{Hyperparameters}
\label{app:hyperparameters}
\begin{table}[H]
    \centering
    \caption{Hyperparameters used in the two-dimensional experiments of Figure~\ref{fig:toy_examples_compact}. The target geometries are defined in Appendix~\ref{app:shapes}. No critic is learned here: $Q$ is specified in closed form. The target $p_\alpha$ is evaluated on a grid only to compute ground-truth mode weights; RFPO itself accesses the target only through $Q$ and its mode weights are estimated from generated samples.}
    \label{tab:hyperparameters_toy}
    \begin{tabular}{
        @{}p{0.36\textwidth}
        p{0.59\textwidth}@{}
    }
        \toprule
        Hyperparameter & Value \\
        \midrule

        Temperature $\alpha$
        & $0.15$ \\

        Action space
        & $[-1,1]^2$, reached as $a=\tanh(u)$ \\

        \midrule

        MALA steps per cycle $K$
        & $10$ \\

        Initial refinement step size $\eta$
        & $5\times10^{-3}$ \\

        Refinement step-size adaptation
        & Multiplicative, targeting an acceptance rate of $0.6$ \\

        Step-size adaptation gain
        & $0.1$ \\

        Refinement step-size bounds
        & $[10^{-6},1]$ \\

        \midrule

        FM steps per cycle
        & $10$ \\

        ODE integration
        & $30$ Euler steps \\

        Integration box
        & $\lVert u\rVert_\infty\leq 10$ \\

        Flow policy network
        & MLP, $3\times128$, SiLU \\

        Flow policy learning rate
        & $10^{-3}$ (Adam) \\

        \midrule

        Training cycles
        & $3{,}000$ \\

        Samples per cycle
        & $1{,}024$ \\

        Pretraining cycles
        & $150$, flow matching onto $\mathcal{N}(0,1.5^2 I)$ \\

        Seeds
        & $4$ \\

        \midrule

        Evaluation samples
        & $8{,}192$ \\

        Evaluation grid
        & $96\times96$ over $[-1,1]^2$ \\

        \bottomrule
    \end{tabular}
\end{table}

\begin{table}[h]
    \centering
    \caption{Hyperparameters used in the six MuJoCo environments.}
    \label{tab:hyperparameters}
    \begin{tabular}{
        @{}p{0.36\textwidth}
        p{0.59\textwidth}@{}
    }
        \toprule
        Hyperparameter & Value \\
        \midrule
        Number of evaluation episodes & 5 \\
        
        Discount factor $\gamma$
        & $0.999$ (Swimmer); $0.99$ (others) \\

        Temperature $\alpha$
        & $0.1$ (HalfCheetah);
          $0.05$ (Walker2d);
          $0.03$ (Ant);
          $0.01$ (Hopper);
          $0.001$ (Swimmer);
          $0.2$ (HumanoidStandup, annealed) \\

        \midrule

        MALA steps per update $K$
        & $5$ \\

        Initial refinement step size $\eta$
        & $10^{-3}$ \\

        Refinement step-size adaptation
        & Multiplicative, targeting an acceptance rate of $0.6$ \\

        Step-size adaptation gain
        & $0.1$ \\

        Refinement step-size bounds
        & $[10^{-8},1]$ \\

        Score norm clip
        & $10$ \\

        \midrule

        FM steps per update
        & $5$ \\

        ODE integration
        & $20$ Euler steps \\

        Integration box
        & $\lVert u\rVert_\infty\leq 10$ \\

        Flow policy network
        & MLP, $3\times512$, SiLU \\

        Time embedding
        & Random Fourier features, dimension $64$ \\

        Flow policy learning rate
        & $3\times10^{-4}$ (Adam) \\

        \midrule

        Critic network
        & MLP, $3\times512$, SiLU \\

        Critic learning rate
        & $3\times10^{-4}$ (Adam) \\

        Target update rate $\tau$
        & $0.005$ \\

        \midrule

        Warm-up steps
        & $5{,}000$, with uniformly sampled actions \\

        Training iterations per environment step
        & $1$ \\

        Replay buffer size
        & $10^6$ \\

        Batch size
        & $256$ \\

        Gradient norm clip
        & $10$ (actor and critic) \\

        \bottomrule
    \end{tabular}
\end{table}

\paragraph{Temperature schedule.}
The temperature is held constant at the value listed in Table~\ref{tab:hyperparameters} for all environments except HumanoidStandup. For HumanoidStandup, a higher initial temperature is used to encourage broader exploration in its high-dimensional action space before gradually increasing the preference for high-value actions. The temperature at environment step $n$ is $\alpha_n=\alpha_{\mathrm{final}}\, 10^{1-\min(1,\,5n/T)}$, where $\alpha_{\mathrm{final}}=0.2$ and $T$ is the total number of environment steps. The temperature therefore starts at $2.0$, reaches $0.2$ after the first $20\%$ of training, and remains constant thereafter.

\paragraph{Refinement step-size adaptation.}
We adapt the refinement step size online to balance the magnitude of sample updates and the acceptance rate. Small step sizes produce limited refinement, whereas excessively large step sizes can lead to frequent rejections. The value listed in Table~\ref{tab:hyperparameters} is the initial step size. After each refinement round of $K$ MALA steps, we update it as
\begin{equation}
    \eta \leftarrow
    \operatorname{clip}\left(
        \eta\exp\left[
            \lambda\left(\hat{p}_{\mathrm{acc}}-p^\star\right)
        \right],
        \eta_{\min},\eta_{\max}
    \right),
    \label{eq:eta_adaptation}
\end{equation}
where $\hat{p}_{\mathrm{acc}}$ is the acceptance-rate estimate from the refinement round, $p^\star=0.6$ is the target acceptance rate, and $\lambda=0.1$ is the adaptation gain. This rule increases the step size when the estimated acceptance rate exceeds $p^\star$ and decreases it otherwise.

The target acceptance rate is chosen as a heuristic, motivated by the asymptotically optimal rate of $0.574$ established for certain high-dimensional target distributions \citep{roberts1998optimal}. The multiplicative update adjusts $\eta$ relative to its current scale, while the adaptation gain limits abrupt changes. The bounds $[\eta_{\min},\eta_{\max}]=[10^{-8},1]$ prevent excessively small or large step sizes. This adaptation is a practical implementation choice; the theoretical analysis considers a fixed refinement kernel.

\paragraph{Score clipping.}
We clip the log-density gradient $\nabla_u\log p_\alpha^u(u\mid s)$ to a Euclidean norm of at most $10$ before using it in the proposal. This limits the drift magnitude in regions with large gradients. Both the forward and reverse proposal densities in the Metropolis--Hastings acceptance ratio are evaluated using the clipped drift, while the target-density ratio uses the original target $p_\alpha^u(\cdot\mid s)$. For a fixed target and step size, the resulting transition therefore satisfies detailed balance with respect to $p_\alpha^u(\cdot\mid s)$: clipping modifies the proposal but preserves the target as an invariant distribution.

\section{Target Geometries}
\label{app:shapes}

All target distributions are defined over $a\in[-1,1]^2$ as
\begin{equation}
    p_\alpha(a)
    \propto
    \exp\left(\frac{Q(a)}{\alpha}\right).
    \label{eq:toy_target_distribution}
\end{equation}
The six landscapes use nominal component amplitudes or height profiles specified using common reference levels $h_{\min}=0.55$ and $h_{\max}=1.00$, with the constructions given below. For any two points $a$ and $b$, their density ratio is $p_\alpha(a)/p_\alpha(b)=\exp((Q(a)-Q(b))/\alpha)$. These common reference levels provide a scale for comparing temperatures across geometries. They do not imply identical component masses, which also depend on component widths, lengths, and overlap. The mode weights reported in the figures are the probability masses of the component regions defined by the assignments below.

\paragraph{Blob targets.}
The \texttt{iso4}, \texttt{aniso4}, and \texttt{grid9} landscapes are constructed as sums of Gaussian-shaped bumps:
\begin{equation}
    Q(a)
    =
    \sum_{i=1}^{M}
    h_i
    \exp\left(
        -\frac{1}{2}
        (a-c_i)^\top\Sigma_i^{-1}(a-c_i)
    \right),
    \label{eq:blob_target}
\end{equation}
where $c_i$, $h_i$, and $\Sigma_i$ specify the center, amplitude, and covariance of bump $i$, respectively. Because the bumps overlap, $h_i$ need not equal the value of $Q$ at $c_i$. For component-wise evaluation, each action is assigned to its nearest center:
\begin{equation}
    i^\star(a)
    =
    \arg\min_i \lVert a-c_i\rVert_2.
    \label{eq:blob_component_assignment}
\end{equation}

\medskip
\noindent\textbf{\texttt{iso4}.}
This landscape contains four isotropic bumps with centers
\begin{equation}
\begin{aligned}
    c_1&=(0.55,0.55), &
    c_2&=(-0.55,0.55), \\
    c_3&=(-0.55,-0.55), &
    c_4&=(0.55,-0.55).
\end{aligned}
\end{equation}
Their amplitudes are $(h_1,h_2,h_3,h_4)=(1.00,0.85,0.70,0.55)$, and all covariances are $\Sigma_i=0.25^2I_2$.

\medskip
\noindent\textbf{\texttt{grid9}.}
This landscape contains nine isotropic bumps with centers $\{c_i\}_{i=1}^{9}=\{-0.62,0,0.62\}^{2}$. The centers are indexed by row from top to bottom and from left to right within each row. The amplitudes are $h_i=1.00-0.45(i-1)/8$ for $i=1,\ldots,9$, and all covariances are $\Sigma_i=0.155^2I_2$.

\medskip
\noindent\textbf{\texttt{aniso4}.}
This landscape uses the same centers and amplitudes as \texttt{iso4}, with anisotropic covariances
\begin{equation}
    \Sigma_i
    =
    R_{\varphi_i}
    \operatorname{diag}(0.28^2,0.10^2)
    R_{\varphi_i}^{\top},
\end{equation}
where $R_{\varphi_i}$ is the two-dimensional rotation matrix with angle $\varphi_i$, and $(\varphi_1,\varphi_2,\varphi_3,\varphi_4)=(0,\pi/4,\pi/2,3\pi/4)$.

\paragraph{Curve targets.}
The \texttt{ring4}, \texttt{spiral4}, and \texttt{arc4} landscapes form ridges around curves represented by points $\{p_j\}_{j=1}^{N}$ with associated heights $\{h_j\}_{j=1}^{N}$. A direct sum of Gaussian bumps would increase in magnitude as more points are added along the curve. We instead use normalized distance-based weights:
\begin{equation}
    w_j(a)
    =
    \frac{
        \exp\left(-\lVert a-p_j\rVert_2^2/(2\tau^2)\right)
    }{
        \sum_{\ell=1}^{N}
        \exp\left(-\lVert a-p_\ell\rVert_2^2/(2\tau^2)\right)
    },
    \label{eq:curve_assignment_weights}
\end{equation}
where $\tau=0.035$ controls the locality of the weighting. The locally averaged height and squared distance are
\begin{equation}
    \bar h(a)=\sum_{j=1}^{N}w_j(a)h_j,
    \qquad
    \bar d^2(a)=
    \sum_{j=1}^{N}w_j(a)\lVert a-p_j\rVert_2^2.
    \label{eq:curve_local_statistics}
\end{equation}
We construct the landscape as
\begin{equation}
    \widetilde Q(a)
    =
    \bar h(a)
    \exp\left(-\frac{\bar d^2(a)}{2\sigma^2}\right),
    \label{eq:unnormalized_curve_target}
\end{equation}
where $\sigma$ controls the ridge width. Since neighboring points receive nonzero weights, $\bar d^2(p_j)$ can be positive even at a sampled curve point. We therefore rescale the landscape as
\begin{equation}
    Q(a)=c_Q\widetilde Q(a),
    \qquad
    c_Q=
    \frac{h_{\max}}{\max_j\widetilde Q(p_j)}.
    \label{eq:curve_target_rescaling}
\end{equation}
This sets the maximum evaluated over the sampled curve points to $h_{\max}=1.00$; it does not enforce the global maximum over the entire action domain.

For component-wise evaluation, an action is assigned to the segment of its nearest sampled curve point:
\begin{equation}
    j^\star(a)=\arg\min_j\lVert a-p_j\rVert_2,
    \qquad
    i^\star(a)=\operatorname{seg}(j^\star(a)).
    \label{eq:curve_nearest_point}
\end{equation}
The segment definitions are given below. Using the sampled curve points avoids assigning spiral points solely by their distances to a small set of representative anchors.

\medskip
\noindent\textbf{\texttt{ring4}.}
The ring is represented by $N=240$ uniformly spaced angles $\theta\in[0,2\pi)$, with $p(\theta)=0.62(\cos\theta,\sin\theta)$. Its height profile is
\begin{equation}
    h(\theta)
    =
    h_0+
    \sum_{i=0}^{3}
    (h_i^\star-h_0)
    \exp\left(
        \lambda[\cos(\theta-\theta_i)-1]
    \right),
    \label{eq:ring_height}
\end{equation}
where $h_0=0.40$, $\lambda=6$, $\theta_i=\pi/4+i\pi/2$, and $(h_0^\star,h_1^\star,h_2^\star,h_3^\star)=(1.00,0.85,0.70,0.55)$. These are nominal peak levels: contributions from neighboring bumps mean that $h(\theta_i)$ need not equal $h_i^\star$ exactly. The four components are angular quarters centered at $\theta_i$, and $\sigma=0.11$.

\medskip
\noindent\textbf{\texttt{spiral4}.}
The spiral is represented by $N=320$ points uniformly spaced in $t\in[0,1]$:
\begin{equation}
\begin{aligned}
    \theta(t)&=4\pi t, \\
    r(t)&=0.16+0.68t, \\
    p(t)&=r(t)(\cos\theta(t),\sin\theta(t)).
\end{aligned}
\end{equation}
The height decreases along the spiral as $h(t)=1.00-0.45t$. The four components correspond to four equal intervals of $t$, and $\sigma=0.09$.

\medskip
\noindent\textbf{\texttt{arc4}.}
This landscape contains four arcs indexed by $i=0,\ldots,3$. Each arc is represented by 60 points uniformly spaced in $\theta$ on $p(\theta)=0.62(\cos\theta,\sin\theta)$ over
\begin{equation}
    \theta\in[\beta_i-0.62,\beta_i+0.62],
    \qquad
    \beta_i=\frac{\pi}{4}+i\frac{\pi}{2}.
\end{equation}
The height is constant within each arc: $h_i=1.00-0.45i/3$. Each arc defines one component, and $\sigma=0.105$.

\paragraph{Design considerations.}
The baseline height $h_0=0.40$ in \texttt{ring4} maintains a continuous ridge between the four higher-value regions. This provides a connected ridge geometry in contrast to the separated high-value regions of \texttt{iso4}.

The curve width parameters are chosen to preserve visible geometric features, such as the central low-density region of the ring. For a local ridge profile $Q(a)=h\exp(-d^2/(2\sigma^2))$, where $d$ is the distance from the ridge, a Taylor expansion near $d=0$ gives
\begin{equation}
    \frac{p_\alpha(a)}{p_\alpha(p)}
    =
    \exp\left[
        \frac{h}{\alpha}
        \left(
            e^{-d^2/(2\sigma^2)}-1
        \right)
    \right]
    \approx
    \exp\left(
        -\frac{h d^2}{2\alpha\sigma^2}
    \right),
    \label{eq:local_ridge_density}
\end{equation}
where $p$ is a point on the ridge with $Q(p)=h$. The local transverse scale is therefore $\sigma\sqrt{\alpha/h}$, showing how the ridge width, height, and temperature jointly determine the concentration around the curve.

\newpage

\section{Proofs}
\label{app:proof}

In this section, we provide the definitions used in the proofs and present the proofs of the results in Section~\ref{sec:theory}.

\subsection{Definitions}
\label{app:defs}

For a random variable $X$, let $\mathrm{Law}(X)$ denote its distribution, and let
$\mathcal{B}(\mathcal{A})$ denote the Borel $\sigma$-algebra on $\mathcal{A}$.

A Markov kernel on $\mathcal{A}$ is a map
$P : \mathcal{A} \times \mathcal{B}(\mathcal{A}) \to [0,1]$
such that $P(a,\cdot)$ is a probability measure for every $a\in\mathcal{A}$,
and $a\mapsto P(a,B)$ is measurable for every $B\in\mathcal{B}(\mathcal{A})$.
For a probability measure $\pi$ on $\mathcal{A}$, let us define
\[
  (\pi P)(B)
  :=
  \int \pi(da)\,P(a,B).
\]
Let $\pi P^K$ denote the measure obtained by applying $P$ $K$ times, and let us call $\pi$
invariant for $P$ if $\pi P = \pi$.

For a fixed state $s$, let $M_s$ denote the kernel given by~\eqref{eq:q_guided_proposal} and~\eqref{eq:q_guided_acceptance} together with the
accept--reject rule, so that $M_s(a,\cdot)$ is the conditional distribution of $a^{(k+1)}$
given $a^{(k)}=a$. Thus, if $a^{(0)}\sim\pi$, then $a^{(K)}\sim\pi M_s^K$.
The Metropolis--Hastings correction ensures that $p_\alpha(\cdot\mid s)$ is invariant
for $M_s$ \citep{roberts1996exponential}.

\subsection{Assumptions}
\label{app:assumptions}

\begin{assumption}[Regularity, moments, and initialization]
\label{ass:reg}
For $\rho$-almost every state $s$, where $\rho$ is the state
distribution used in the self-target FM objective, the following
conditions hold throughout the cycles under analysis:
\begin{enumerate}
    \item[(i)] The normalizing constant $Z_\alpha(s)$ in
    \eqref{eq:target_action_distribution} satisfies
    $0<Z_\alpha(s)<\infty$.

    \item[(ii)] For every $n$,
    $\mathbb{E}_{a\sim\pi_n(\cdot\mid s)}
    [|Q_\phi(s,a)|]<\infty$.

    \item[(iii)] In every cycle, the refined action satisfies
    $\mathbb{E}[\|a_1'\|^2\mid s]<\infty$.

    \item[(iv)] In every cycle, $v^\star(s,x,t)$ is jointly
    continuous in $x$ and $t$ for $t\in[0,1)$. For every
    $0<T<1$, there is a constant $L_T$ such that
    $v^\star(s,\cdot,t)$ is $L_T$-Lipschitz for all
    $t\in[0,T]$.

    \item[(v)] The initial policy satisfies
    $\mathrm{KL}\bigl(
        \pi_0(\cdot\mid s)\,\|\,p_\alpha(\cdot\mid s)
    \bigr)<\infty$.
    \item[(vi)] $Q_\phi(s,\cdot)$ and $\nabla_a Q_\phi(s,\cdot)$ are finite on the state-wise refinement domain.
\end{enumerate}
\end{assumption}

\subsection{Integrability of the Velocity Field}
\label{app:integrability}

\begin{lemma}[Integrability of the velocity field]
\label{lem:vstar-bounds}
Fix a state $s$ and let Assumption~\ref{ass:reg} hold. Then for each $t < 1$ the velocity field
$v^\star(s,\cdot,t)$ is defined $p_t(\cdot \mid s)$-almost everywhere and satisfies
\[
  \int \bigl\| v^\star(s,x,t) \bigr\| \, p_t(dx \mid s) \leq \mathbb{E} \bigl\| a_1' - x_0 \bigr\| ,
  \qquad
  \int \bigl\| v^\star(s,x,t) \bigr\|^2 \, p_t(dx \mid s) \leq \mathbb{E} \bigl\| a_1' - x_0 \bigr\|^2 .
\]
Consequently
\[
  \int_0^1 \!\! \int \bigl\| v^\star(s,x,t) \bigr\| \, p_t(dx \mid s) \, dt < \infty ,
  \qquad
  \int_0^1 \!\! \int \bigl\| v^\star(s,x,t) \bigr\|^2 \, p_t(dx \mid s) \, dt < \infty .
\]
\end{lemma}

\begin{proof}
Fix $s$ and omit it from the notation, and let $t < 1$. The field $v^\star(\cdot,t)$ is a conditional
expectation of $a_1' - x_0$ given $a_t$, so it is determined up to a $p_t$-null set, and we fix one such
version throughout. Conditional Jensen's inequality applied to the norm gives
\[
  \int \bigl\| v^\star(x,t) \bigr\| \, p_t(dx)
  = \mathbb{E} \bigl\| \mathbb{E}\bigl[\, a_1' - x_0 \,\big|\, a_t \,\bigr] \bigr\|
  \leq \mathbb{E} \bigl\| a_1' - x_0 \bigr\| ,
\]
and the same argument applied to the squared norm gives the second bound. Both right-hand sides are finite because $x_0$ is Gaussian and $\mathbb{E}\|a_1'\|^2 < \infty$ by Assumption~\ref{ass:reg}. Neither bound depends on $t$, so integrating over $t \in [0,1]$ preserves finiteness.
\end{proof}

\subsection{Proof of Lemma~\ref{lem:coupling}}
\label{proof:lemma1}

\begin{proof}
 Since $\mathcal{L}_{\mathrm{SFM}}$ is the average
over $s \sim \rho$ of the corresponding state-wise objective, it suffices to minimize the
latter for each $s$. Fix $s$ and omit it from the notation. Let us define $Y := a_1' - x_0$ and let $\Phi_t^{v^\star}$ denote the flow map induced by $v^\star$
from time 0 to time $t$.

By Assumption~\ref{ass:reg} and the Gaussianity of $x_0$ we have $\mathbb{E}\|Y\|^2 < \infty$, and
Lemma~\ref{lem:vstar-bounds} gives $\int_0^1\!\int \|v^\star\|^2 \, dp_t \, dt \leq \mathbb{E}\|Y\|^2$.
These moment conditions are all that the marginal-preservation results for conditional flow
matching under a general coupling of the source and the target require
\citep[Lemma~3.1]{pooladian2023multisample}, \citep[Theorems~3.1 and 3.2]{tong2024improving}. The pair $(x_0, a_1')$ is such a coupling of
$\mathcal{N}(0, I_d)$ and $\mu_n$, with $a_1'$ obtained from $x_0$ through
\eqref{eq:flow_endpoint}, \eqref{eq:q_guided_proposal}, and \eqref{eq:q_guided_acceptance}. Applied to
this coupling, those results give two facts. First, for any $v$ with
$\int_0^1\!\int \|v\|^2 \, dp_t \, dt < \infty$,
\[
  \mathbb{E}\bigl\| v(a_t,t) - Y \bigr\|^2
  = \int_0^1 \!\int \bigl\| v - v^\star \bigr\|^2 \, dp_t \, dt
  + \mathbb{E}\bigl\| v^\star(a_t,t) - Y \bigr\|^2 ,
\]
since the cross term vanishes by the tower property of conditional expectation. The second
term does not depend on $v$, so $v^\star$ minimizes the objective and every minimizer agrees
with $v^\star$ for $dt \otimes p_t$-almost every $(t, x)$. Second, $(p_t)_{t \in [0,1)}$ is a
weak solution of
\[
  \partial_t p_t + \nabla \cdot \bigl( p_t v^\star \bigr) = 0, \qquad p_0 = \mathcal{N}(0, I_d).
\]

The remainder of the proof is specific to our setting. Lemma~\ref{lem:chain} requires
the flow map generated by $v^\star$ to exist and to reach $t = 1$, and Assumption~\ref{ass:reg}
is what makes this possible.

Fix $\tau < 1$. The family $(p_t)_{t \leq \tau}$ is narrowly continuous, since $a_t$ is continuous in $t$ and
dominated convergence applies to $\int f \, dp_t$ for every bounded continuous $f$. Lemma~\ref{lem:vstar-bounds}
gives $\int_0^\tau \!\int \|v^\star\| \, dp_t \, dt < \infty$.
Moreover, for every compact set $B \subset \mathbb{R}^d$,
joint continuity of $v^\star$ implies
$\sup_{(x,t)\in B\times[0,\tau]}\|v^\star(x,t)\| < \infty$,
while Assumption~\ref{ass:reg} gives
$\operatorname{Lip}(v^\star(\cdot,t);B) \leq L_\tau$ for all
$t\in[0,\tau]$. Hence,
\[
  \int_0^\tau \left(
    \sup_{x\in B}\|v^\star(x,t)\|
    + \operatorname{Lip}(v^\star(\cdot,t);B)
  \right) dt
  \;<\; \infty .
\]

The field therefore satisfies the
regularity and integrability conditions required by the representation theorem for the continuity equation
\citep[Proposition 8.1.8]{ambrosio2005gradient}. That theorem applies to
$(p_t)_{t \leq \tau}$ and shows that the characteristic equation
$\dot z(t) = v^\star(z(t),t)$ with $z(0) = x_0$ admits a solution $z(t) = \Phi^{v^\star}_t(x_0)$ on all of $[0,\tau]$ for $\mathcal{N}(0,I_d)$-almost every $x_0$, and that $p_t$ is the law of that solution. Since
$\tau < 1$ is arbitrary,
\[
  \mathrm{Law}\bigl( \Phi_t^{v^\star}(x_0) \bigr) = p_t, \qquad t < 1 .
\]

The identity above and Lemma~\ref{lem:vstar-bounds} give
\[
  \mathbb{E} \int_0^1 \|\dot z(t)\| \, dt
  = \int_0^1 \int \|v^\star(x,t)\| \, p_t(dx) \, dt < \infty ,
\]

so $z$ admits an absolutely continuous extension to $[0,1]$ with probability one and
$\Phi^{v^\star}(x_0) = \lim_{t \to 1} z(t)$ exists with probability one. Therefore
\[
  \mathrm{Law}\bigl( \Phi^{v^\star}(x_0) \bigr)
  = \lim_{t \to 1} \mathrm{Law}\bigl( \Phi_t^{v^\star}(x_0) \bigr)
  = \lim_{t \to 1} p_t
  = \mu_n ,
\]
where the limits are weak, the first holds because almost sure convergence implies
convergence in distribution, the second is the identity above, and the third holds because
$a_t \to a_1'$ with probability one.
\end{proof}

\subsection{Proof of Lemma~\ref{lem:chain}}
\label{proof:lemma2}

\begin{proof}
Fix $s$ and omit it from the notation. By~\eqref{eq:flow_endpoint},
$a_1 = \Phi_{\theta_n}(x_0)$ with $x_0 \sim \mathcal{N}(0, I_d)$, so $a_1 \sim \pi_n$.
Since $a^{(0)} = a_1$ and the refined action is $a_1' = a^{(K)}$, Appendix~\ref{app:defs}
gives
\[
  \mu_n = \pi_n M_s^{K}.
\]

Lemma~\ref{lem:coupling} determines the minimizer only up to a $dt \otimes p_t$-negligible
set. This suffices here. Every nonempty open set receives positive mass under $p_t$ for each $t < 1$. To see
this, let $\kappa_s(a, da') = q_\phi(a' \mid a)\, P_{\mathrm{acc}}(a, a')\, da'$ denote the sub-kernel
corresponding to an accepted move, so that $M_s \geq \kappa_s$ and hence $M_s^K \geq \kappa_s^K$. The density
of $\kappa_s$ is positive on all of $\mathbb{R}^d \times \mathbb{R}^d$, since the proposal is Gaussian and
$Q_\phi(\cdot)$ and $\nabla_a Q_\phi(\cdot)$ are finite, and the $K$-fold density of $\kappa_s^K$ is
therefore positive as well. Conditioning on $x_0 = x$ fixes $a^{(0)} = \Phi_{\theta_n}(x)$, so the law of
$a_1'$ given $x_0 = x$ dominates a measure with everywhere positive density. For $t \in (0,1)$ the interpolated action
is an affine bijection of $a_1'$ at fixed $x$, so the same holds for the law of $a_t$ given $x_0 = x$, and
integrating over $x_0$ preserves the property, while the case $t = 0$ is immediate from
$p_0 = \mathcal{N}(0, I_d)$. This argument rests on the structure of the refinement kernel rather than on any
independence between $x_0$ and $a_1'$. Both $v_{\theta_{n+1}}$ and $v^\star$ are continuous on
$\mathbb{R}^d \times [0,1)$, the former because it is a neural network and the latter by
Assumption~\ref{ass:reg}. Two continuous fields that agree $dt \otimes p_t$-almost everywhere
therefore agree everywhere on $\mathbb{R}^d \times [0,1)$, so the exact inner solution gives
$v_{\theta_{n+1}} = v^\star$. Since the flow is integrated exactly,
$\Phi_{\theta_{n+1}}(x_0) = \Phi^{v^\star}(x_0)$. The learned field therefore inherits the regularity
granted to $v^\star$ by Assumption~\ref{ass:reg}, so no separate condition on it is required.

Therefore
\[
  \pi_{n+1}
  = \mathrm{Law}\bigl( \Phi^{v^\star}(x_0) \bigr)
  = \mu_n
  = \pi_n M_s^{K}.
\]
The second equality holds due to Lemma~\ref{lem:coupling}.
Iterating this identity from $\pi_0$ yields $\pi_n = \pi_0 M_s^{nK}$.
\end{proof}

\subsection{Proof of Theorem~\ref{thm:improve}}
\label{proof:monotone_improvement}

\begin{proof}
Fix $s$ and omit it from the notation. Let us first relate $\mathcal{J}$ to a divergence
from the target. For any $\pi$ with a density,
\[
  \begin{aligned}
    \mathrm{KL}\bigl( \pi \,\|\, p_\alpha \bigr)
    &= \int \pi(a) \log \frac{\pi(a)}{p_\alpha(a)} \, da \\
    &= -\mathcal{H}(\pi)
       - \frac{1}{\alpha} \mathbb{E}_{a \sim \pi}\bigl[ Q_\phi(a) \bigr]
       + \log Z_\alpha \\
    &= \log Z_\alpha - \frac{1}{\alpha} \mathcal{J}(\pi),
  \end{aligned}
\]
where the second equality substitutes the definition of $p_\alpha$ and the third collects
terms. The constant $Z_\alpha$ is finite and positive by Assumption~\ref{ass:reg}, so the identity is
meaningful for every $\pi$ that admits a density.

Assumption~\ref{ass:reg} makes $\mathrm{KL}(\pi_0 \,\|\, p_\alpha)$ finite, so $\pi_0 \ll p_\alpha$ and
$\pi_0$ admits a density. The data-processing inequality below propagates finiteness to every $n$, and
Lemma~\ref{lem:chain} then gives $\pi_n \ll p_\alpha$ for every $n$, so each $\pi_n$ admits a density and the
identity above applies to it. Rearranging that identity expresses $-\mathcal{H}(\pi_n)$ as a sum of finite
terms, using $\mathbb{E}_{a \sim \pi_n}[\,|Q_\phi(a)|\,] < \infty$ from Assumption~\ref{ass:reg}, so
$\mathcal{H}(\pi_n)$ and $\mathcal{J}(\pi_n)$ are finite for every $n$.

By Appendix~\ref{app:defs}, $p_\alpha M_s = p_\alpha$, and hence $p_\alpha M_s^K = p_\alpha$.
Applying the data-processing inequality for the KL divergence to the kernel
$M_s^K$ gives
\[
  \mathrm{KL}\bigl( \pi_n M_s^K \,\|\, p_\alpha M_s^K \bigr)
  \;\leq\; \mathrm{KL}\bigl( \pi_n \,\|\, p_\alpha \bigr),
\]
and Lemma~\ref{lem:chain} identifies the left-hand side with
$\mathrm{KL}( \pi_{n+1} \,\|\, p_\alpha )$. Combining this with the identity above,
\[
  \log Z_\alpha - \frac{1}{\alpha} \mathcal{J}(\pi_{n+1})
  \;\leq\;
  \log Z_\alpha - \frac{1}{\alpha} \mathcal{J}(\pi_n),
\]
which rearranges to $\mathcal{J}(\pi_{n+1}) \geq \mathcal{J}(\pi_n)$ since $\alpha > 0$.
Neither step constrains $\eta$ or $K$. If $\pi_n = \pi_n M_s^K$, then
$\pi_{n+1} = \pi_n$ by Lemma~\ref{lem:chain} and the two objectives coincide.

Finally, the identity above shows that maximizing $\mathcal{J}$ is equivalent to minimizing
$\mathrm{KL}( \pi \,\|\, p_\alpha )$, which is nonnegative and vanishes only at
$\pi = p_\alpha$. Taking $\pi_n = p_\alpha$ gives $\pi_{n+1} = p_\alpha M_s^K = p_\alpha$,
so $p_\alpha$ is a fixed point of the outer cycle.
\end{proof}

\subsection{Bounded Actions}
\label{app:bounded_proof}

The bounded-action implementation in Appendix~\ref{app:bounded_actions}
operates in the unconstrained coordinate $u\in\mathbb{R}^d$ and targets
the induced density $p_\alpha^u(\cdot\mid s)$ in
\eqref{eq:explicit_bounded_target}. The following extension concerns
the idealized latent-coordinate formulation: the refinement kernel is
held fixed and the flow ODE is integrated exactly. Thus, it does not
by itself cover Euler discretization or the latent clipping used during
practical flow sampling.

For a fixed target and step size, the score-clipped proposal described
in Appendix~\ref{app:hyperparameters} still leaves
$p_\alpha^u(\cdot\mid s)$ invariant, because the same clipped drift is
used in both the forward and reverse proposal densities. It can
therefore be viewed as an alternative fixed refinement kernel in the
analysis. In contrast, the online step-size adaptation changes the
kernel across updates and is outside the fixed-kernel setting of
Theorem~\ref{thm:improve}.

The normalizing constant of $p_\alpha^u(\cdot\mid s)$ is finite when
$Q_\phi(s,\cdot)$ is bounded on the bounded action space. Indeed,
\begin{equation}
    \left|\det J_T(u)\right|
    =
    \prod_{j=1}^{d}
    \Delta a_j
    \left(1-\tanh^2(u_j)\right)
\end{equation}
decays exponentially in the tails and is integrable on
$\mathbb{R}^d$. In contrast, its logarithm decreases linearly in the
tails; it is the Jacobian determinant, rather than the log-Jacobian,
that decays exponentially.

To relate the latent and action-coordinate objectives, let
$\pi^u(\cdot\mid s)$ be a distribution on $\mathbb{R}^d$ with a
density, and let $\pi^a(\cdot\mid s)=T_{\#}\pi^u(\cdot\mid s)$ denote
its pushforward through $T$. Define the latent-coordinate objective as
\begin{equation}
\begin{aligned}
    \mathcal{J}^u(\pi^u\mid s)
    &:=
    \mathbb{E}_{u\sim\pi^u(\cdot\mid s)}
    \left[
        Q_\phi(s,T(u))
        +
        \alpha\log\left|\det J_T(u)\right|
    \right] \\
    &\quad+
    \alpha\mathcal{H}(\pi^u(\cdot\mid s)).
\end{aligned}
\label{eq:latent_policy_objective}
\end{equation}

\begin{proposition}[Coordinate equivalence]
\label{prop:bounded_coordinate_equivalence}
For every fixed state $s$ and every $\pi^u$ satisfying the stated
integrability conditions,
\begin{equation}
    \mathrm{KL}\left(
        \pi^u(\cdot\mid s)\,\middle\|\,p_\alpha^u(\cdot\mid s)
    \right)
    =
    \mathrm{KL}\left(
        \pi^a(\cdot\mid s)\,\middle\|\,p_\alpha(\cdot\mid s)
    \right),
\end{equation}
and
\begin{equation}
    \mathcal{J}^u(\pi^u\mid s)
    =
    \mathcal{J}(\pi^a\mid s).
\end{equation}
\end{proposition}

\begin{proof}
The density of $\pi^a$ at $T(u)$ is
$\pi^u(u)/|\det J_T(u)|$, and the same change-of-variables relation
holds between $p_\alpha^u$ and $p_\alpha$. The Jacobian factors
therefore cancel in the likelihood ratio, yielding the KL identity.
Moreover,
\begin{equation}
    \mathcal{H}(\pi^a)
    =
    \mathcal{H}(\pi^u)
    +
    \mathbb{E}_{u\sim\pi^u}
    \left[
        \log\left|\det J_T(u)\right|
    \right].
\end{equation}
Substituting this identity into \eqref{eq:latent_policy_objective}
gives $\mathcal{J}^u(\pi^u\mid s)=\mathcal{J}(\pi^a\mid s)$.
\end{proof}

Accordingly, Theorem~\ref{thm:improve} applies to the idealized
latent-coordinate update and yields the same KL and
entropy-regularized-objective conclusions for the induced action
distribution $\pi^a$. The practical implementation recovers bounded
actions as $T(u)$.

\section{Baselines}
\label{app:baselines}
We reproduce the baselines by following the authors' public implementations, keeping their default hyperparameters for all experiments. For evaluation, we follow the action-generation procedure used in each baseline's public implementation.

\paragraph{SAC.}
Soft Actor--Critic (SAC) \citep{haarnoja2018soft} is an off-policy actor--critic algorithm based on maximum-entropy reinforcement learning. It uses a squashed Gaussian policy and optimizes expected Q-values together with policy entropy. The critic is trained using transitions sampled from a replay buffer. SAC serves as a representative Gaussian-policy baseline for evaluating the performance of RFPO.

\paragraph{DPMD.}
Diffusion Policy Mirror Descent (DPMD) \citep{ma2025efficient} trains a diffusion policy to approximate a policy mirror descent update. This update favors high-Q actions while penalizing KL divergence from the previous policy. Its target distribution is proportional to the previous policy density multiplied by an exponential Q-value weight. DPMD uses reweighted score matching to learn this distribution without requiring direct samples from it.

\paragraph{SDAC.}
Soft Diffusion Actor--Critic (SDAC) \citep{ma2025efficient} uses the same reweighted score-matching framework to train a diffusion policy for maximum-entropy policy optimization. Its policy target is a Boltzmann distribution induced by the Q-function. The resulting training objective avoids direct sampling from this target and backpropagation through the entire reverse diffusion process. SDAC provides a comparison with a diffusion-based method that also learns a Q-induced target distribution.

\paragraph{DIPO.}
Diffusion Policy Optimization (DIPO) \citep{yang2023policy} is an online RL algorithm that alternates between action improvement and diffusion policy learning. It refines replay-buffer actions through gradient ascent on the learned Q-function and trains a diffusion policy on the resulting actions using a denoising objective. DIPO therefore provides a comparison with an approach that explicitly constructs policy-training targets through action refinement. Whereas DIPO starts from replay-buffer actions and uses Q-gradient ascent, RFPO starts from current-policy actions and uses MALA before self-target FM.

\paragraph{RFM.}
Reverse Flow Matching (RFM) \citep{li2026reverse} trains a flow policy toward a Boltzmann distribution induced by the Q-function without requiring direct samples from that distribution. It formulates velocity-target construction as a posterior mean estimation problem conditioned on intermediate noisy samples. These targets are estimated using self-normalized importance sampling with control variates derived from Langevin Stein operators, combining Q-value and Q-gradient information. RFM provides a comparison with a Q-guided flow-policy method that directly estimates velocity targets, whereas RFPO constructs refined action targets through MALA and trains the policy through self-target FM.

% \clearpage
\section{Full Results on Multimodal Target Distributions}
\label{app:toy_full_results}
This section presents the full results of the experiments in Section~\ref{sec:multimodal_distributions}, showing generated samples at selected cycles and the evolution of component-region masses over 3,000 cycles. In the component-region-mass curves, dashed lines indicate ground-truth masses, and shaded bands indicate standard deviations across four random seeds.

\begin{figure}[H]
    \centering
    \includegraphics[width=\linewidth]{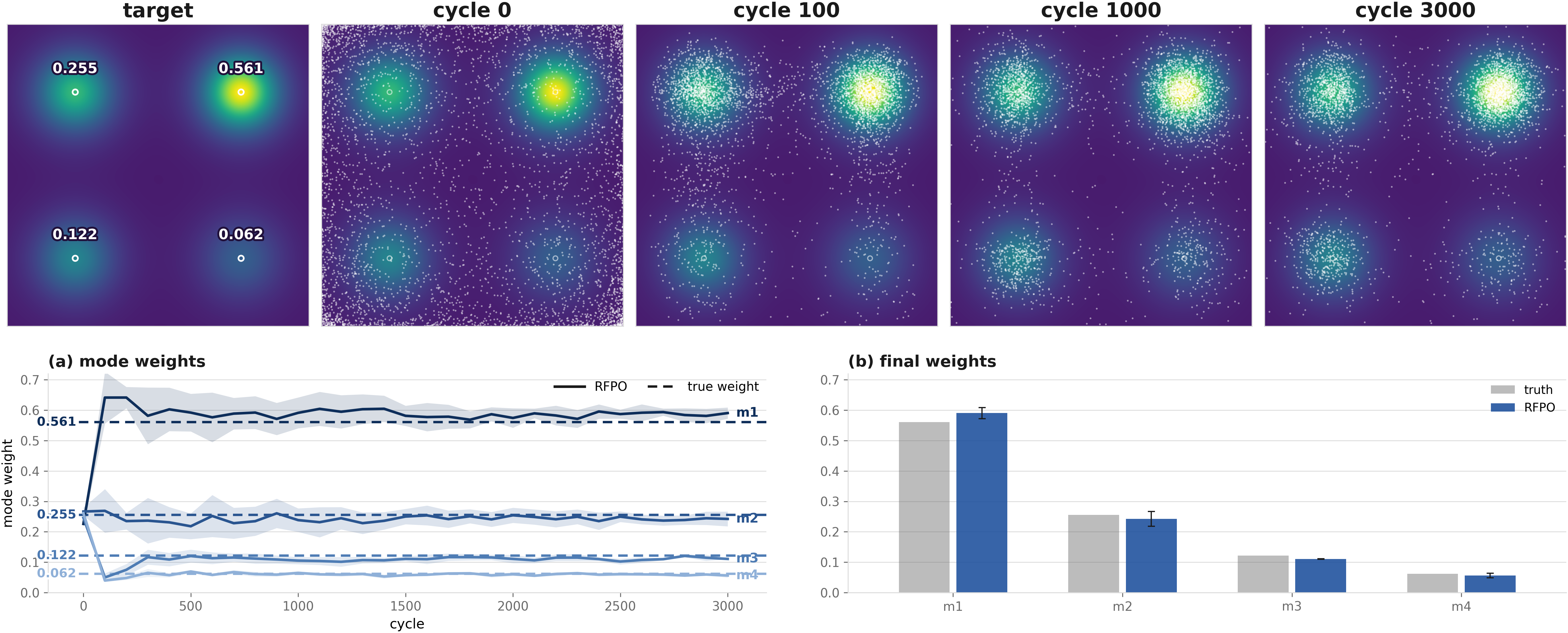}
    \caption{Evolution of the generated sample distribution and mode weights for the isotropic target.}
    \label{fig:toy_2d}
\end{figure}

\begin{figure}[H]
    \centering
    \includegraphics[width=\linewidth]{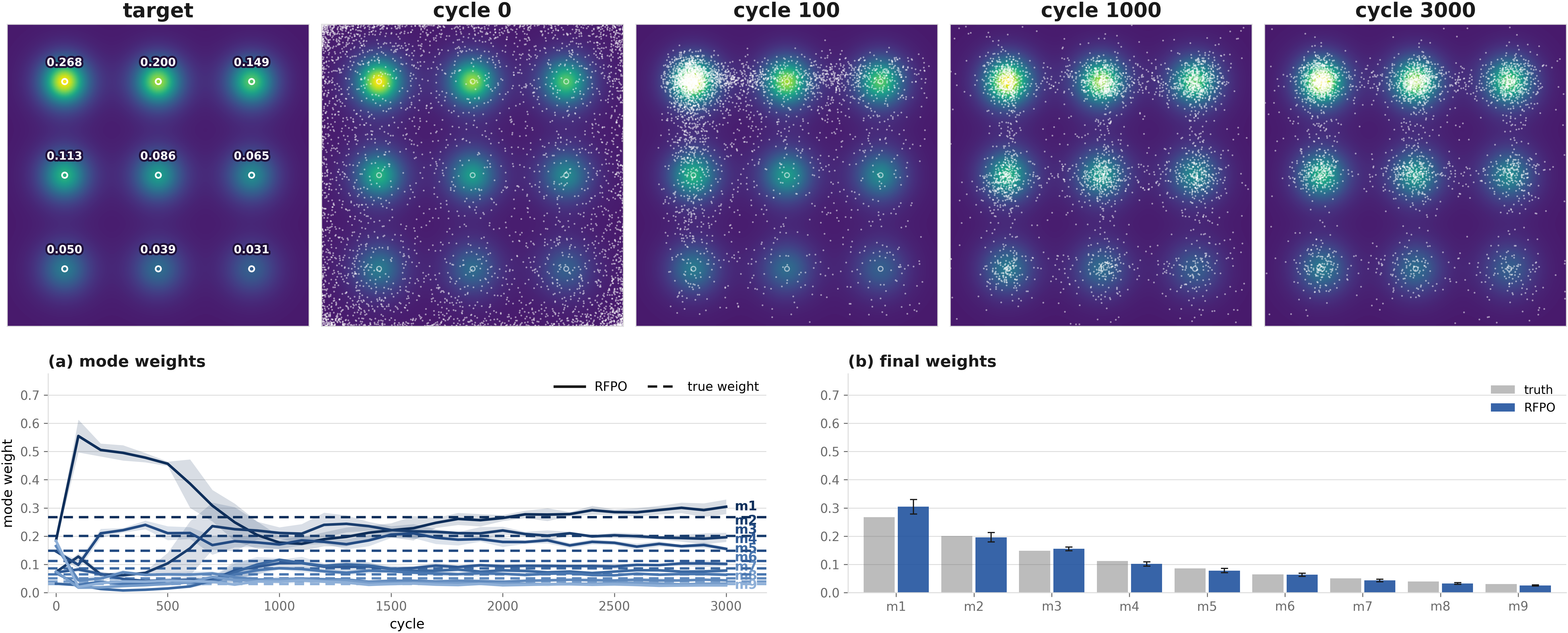}
    \caption{Evolution of the generated sample distribution and mode weights for the grid target.}
    \label{fig:toy_grid}
\end{figure}

\begin{figure}[H]
    \centering
    \includegraphics[width=\linewidth]{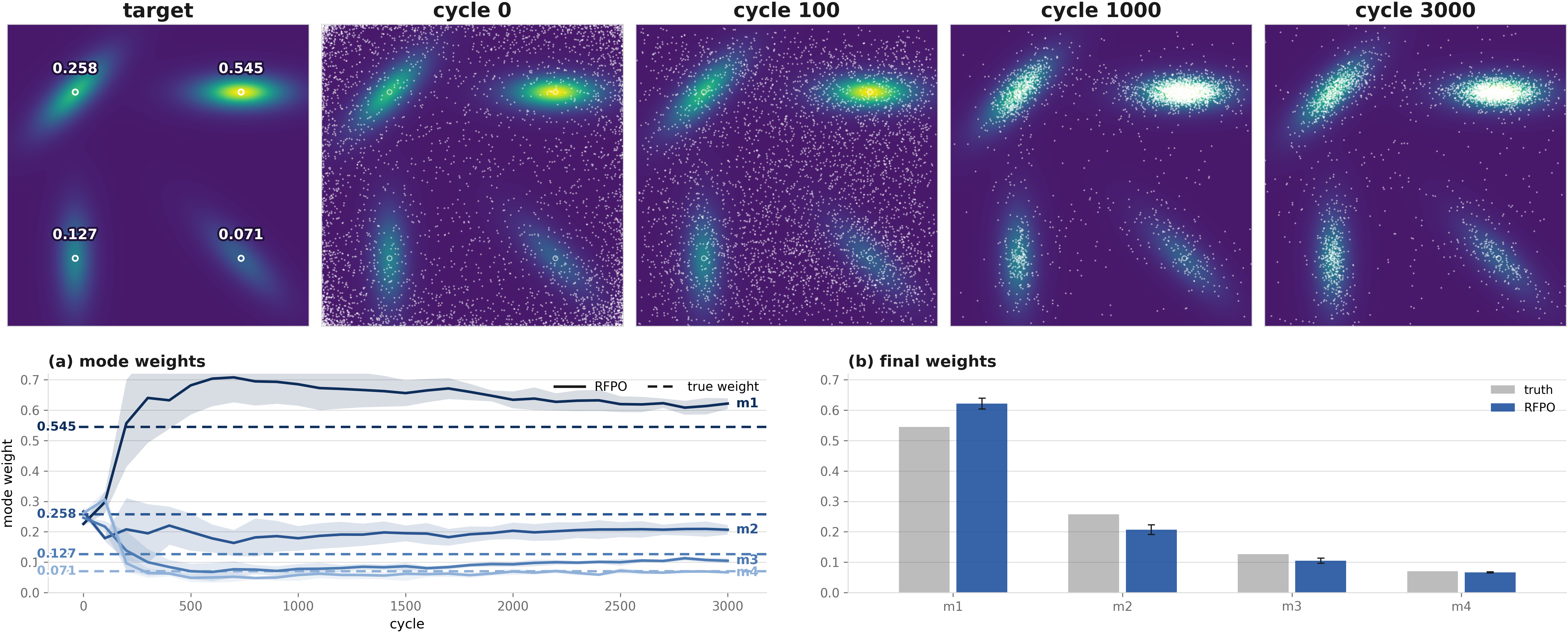}
    \caption{Evolution of the generated sample distribution and mode weights for the anisotropic target.}
    \label{fig:toy_aniso}
\end{figure}

\begin{figure}[H]
    \centering
    \includegraphics[width=\linewidth]{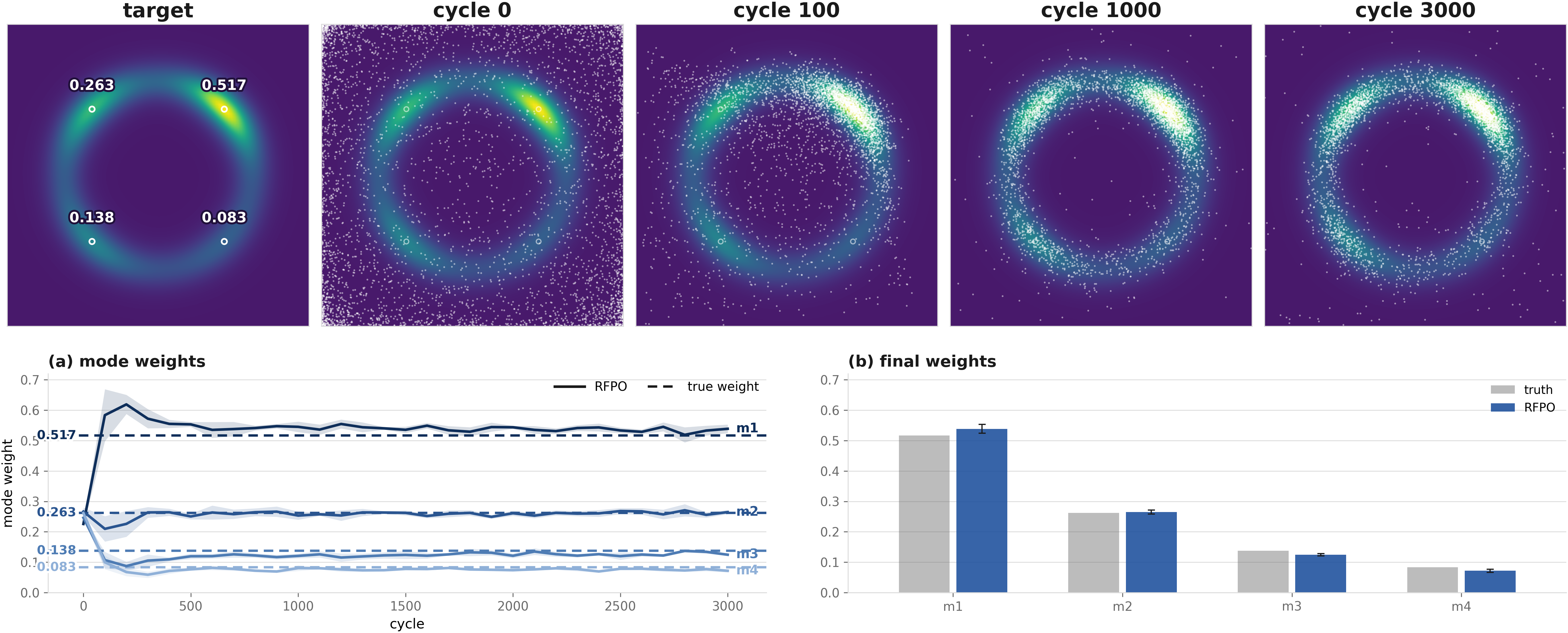}
    \caption{Evolution of the generated sample distribution and mode weights for the ring-shaped target.}
    \label{fig:toy_ring}
\end{figure}

\begin{figure}[H]
    \centering
    \includegraphics[width=\linewidth]{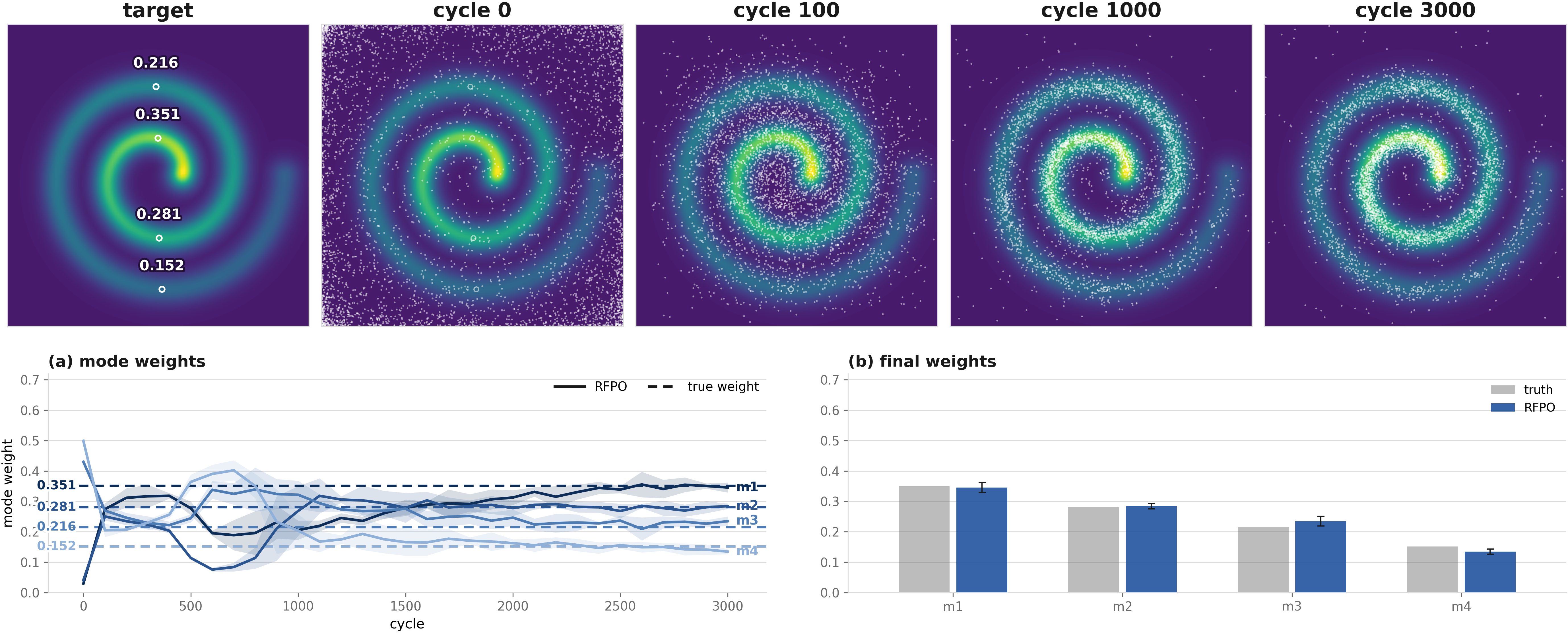}
    \caption{Evolution of the generated sample distribution and mode weights for the spiral-shaped target.}
    \label{fig:toy_spiral}
\end{figure}

\begin{figure}[H]
    \centering
    \includegraphics[width=\linewidth]{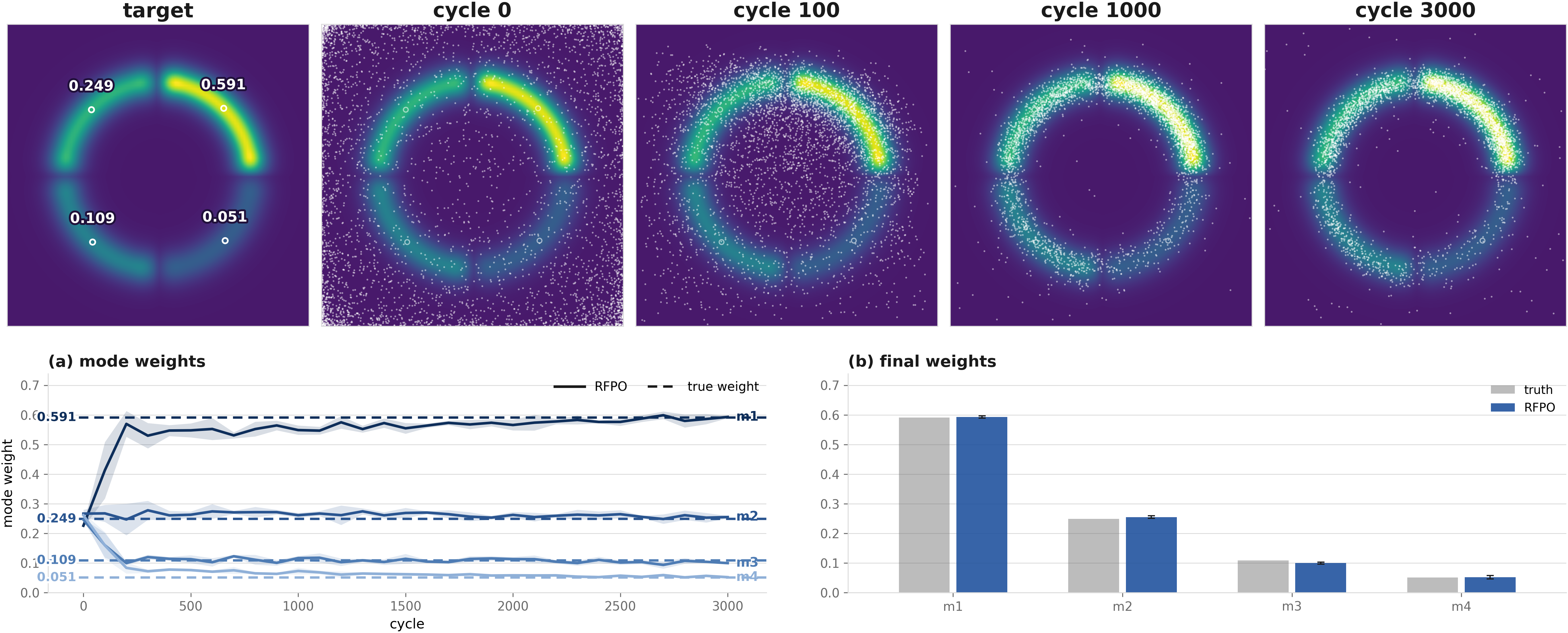}
    \caption{Evolution of the generated sample distribution and mode weights for the arc-shaped target.}
    \label{fig:toy_arc}
\end{figure}

\section{Computational Cost}
\label{app:compute}

Refinement and self-target flow matching introduce additional computation into each policy update. We therefore compare the implementation-level network cost of each method by counting network invocations and the number of input rows processed per environment step. Counting rows is important because a network invocation on $25{,}600$ rows is not comparable to one on $256$ rows, and the methods differ substantially in this respect.

Table~\ref{tab:compute} follows the protocol of Section~\ref{sec:continuous_control}: one training iteration per environment step, batch size $256$, twin critics counted as two networks, and one action sampled per environment step. The reported counts use Hopper-v4. Apart from DIPO's environment-specific action-ascent length, the call and row counts do not depend on the observation or action dimension, although the per-row computational cost does. The policy and critic columns report forward-pass calls and input rows. Total rows additionally includes backward passes, with each forward or backward pass contributing its batch size once, without weighting. The GFLOP estimates account for network architecture and approximate backward-pass costs: a backward pass that computes parameter gradients is counted as two forward-pass equivalents, whereas one that computes input gradients only is counted as one. Input-gradient-only backward passes include those used for the MALA score in RFPO, the action-ascent update in DIPO, and the posterior estimator in RFM. These are analytical estimates of network computation rather than wall-clock measurements.

\begin{table}[H]
    \centering
    \caption{Network cost per environment step. Within the policy and critic columns, “calls” counts forward network invocations and “rows” counts the corresponding input rows. “Total rows” includes both forward and backward passes without weighting, and $\times$SAC normalizes this total by SAC. The GFLOP estimates account for network architecture and the backward-pass weights described in the text. RFPO uses $3\times512$ trunks, whereas the baselines use hidden units of width $256$; consequently, the FLOP comparison differs from the row-count comparison.}
    \label{tab:compute}
    \small
    \begin{tabular}{@{}lrrrrrrr@{}}
        \toprule
        & \multicolumn{2}{c}{Policy network} & \multicolumn{2}{c}{Critic networks}
        & \multirow{2}{*}{\shortstack[r]{Total\\rows}} & \multirow{2}{*}{$\times$SAC}
        & \multirow{2}{*}{\shortstack[r]{GFLOP\\($\times$SAC)}} \\
        \cmidrule(lr){2-3} \cmidrule(lr){4-5}
        & calls & rows & calls & rows & & & \\
        \midrule
        SAC   &   3 &     513 &  6 &  1{,}536 &   3{,}329 &   1.0 & 0.6\ \ (1.0) \\
        RFPO  &  65 & 11{,}540 & 16 &  4{,}096 &  20{,}500 &   6.2 & 24.7\ (43) \\
        DIPO  & 201 & 25{,}956 & 44 & 11{,}264 &  48{,}228 &  14.5 & 14.6\ (26) \\
        DPMD  &  41 & 164{,}736 & 10 & 17{,}984 & 183{,}488 &  55.1 & 51.6\ (91) \\
        RFM   &  21 & 164{,}416 & 10 & 85{,}056 & 301{,}440 &  90.5 & 85.3\ (150) \\
        SDAC  &  61 & 344{,}704 & 12 & 66{,}624 & 428{,}224 & 128.6 & 124.4\ (218) \\
        \bottomrule
    \end{tabular}
\end{table}

\paragraph{Where the cost goes.}
RFPO's two $20$-step ODE integrations---one to generate the endpoint for refinement and one to generate the next action for the critic target---account for $10{,}240$ of its $20{,}500$ total rows. The five-step MALA refinement accounts for $6{,}144$ rows, or $30\%$ of the total. Each refinement uses six critic-ensemble evaluations and six input-gradient computations at batch size $256$.

The baselines allocate computation differently. DIPO uses a $100$-step reverse chain and a $20$-step action-ascent procedure. DPMD samples one best-of-$32$ chain at batch size $8{,}192$, accounting for $163{,}840$ rows, or $89\%$ of its total cost. SDAC draws this chain twice, accounting for most of its additional cost relative to DPMD. RFM integrates $10$ Euler steps on $16{,}384$ rows and then evaluates the twin critic, including input-gradient computations, on $25{,}600$ rows for posterior-mean estimation.

The widest tensor processed by an RFPO network has the training batch size of $256$. In contrast, the particle-based methods expand this dimension to $8{,}192$ for DPMD and SDAC, and to $16{,}384$ and then $25{,}600$ for RFM. RFPO instead uses six short critic input-gradient computations to refine each policy-generated action.

\end{document}